\documentclass{article}

\usepackage{arxiv}
\usepackage[utf8]{inputenc} 
\usepackage{hyperref}       
\usepackage{booktabs}       
\usepackage{microtype}      
\usepackage{lipsum}
\usepackage{graphicx}
\usepackage[T1]{fontenc} 
\usepackage{url} 
\usepackage{amsfonts} 
\usepackage{nicefrac} 
\graphicspath{ {./images/} }
\usepackage{natbib}
\usepackage{mathptmx} 
\usepackage{times} 
\usepackage{amsmath} 
\usepackage{amssymb} 

\usepackage{amsthm}
\usepackage{amsmath}

\usepackage{color}
\usepackage{multirow}
\usepackage{algorithm}
\usepackage{algorithmic}
\usepackage{bm}
\usepackage{bbm}
\usepackage{float} 
\usepackage[cal=cm]{mathalfa} 
\usepackage{xcolor}
\usepackage{subfigure}

\newtheorem{proposition}{Proposition}
\newtheorem{theorem}{Theorem}
\newtheorem{lemma}{Lemma}
\newtheorem{corollary}{Corollary}
\newtheorem{assumption}{Assumption}
\allowdisplaybreaks[4]

\def\br#1{\color{brown}{#1}}
\def\bl#1{\color{blue}{#1}}

\title{Distributed TD(0) with Almost No Communication}

\author{
  Rui Liu \\
  Division of Systems Engineering\\
  Boston University\\
  Boston, MA, 02215 \\
  \texttt{rliu@bu.edu} \\
   \And
 Alex Olshevsky \\
  Department of ECE and Division of Systems Engineering\\
  Boston University\\
  Boston, MA, 02215 \\
  \texttt{alexols@bu.edu} \\
}

\begin{document}
\maketitle
\begin{abstract} We provide a new non-asymptotic analysis of distributed TD(0) with linear function approximation. Our approach relies on ``one-shot averaging,'' where $N$ agents run local copies of TD(0) and average the outcomes only once at the very end. We consider two models: one in which the agents interact with an environment they can observe and whose transitions depends on all of their actions (which we call the global state model), and one in which each agent can run a local copy of an identical Markov Decision Process, which we call the local state model. 

In the global state model, we show that the convergence rate of our distributed one-shot averaging method matches the known convergence rate of TD(0). By contrast, the best convergence rate in the previous literature showed a rate which, according to the worst-case bounds given, could underperform the non-distributed version by $O(N^3)$ in terms of the number of agents $N$. In the local state model, we demonstrate a version of the linear time speedup phenomenon, where the convergence time of the distributed process is a factor of $N$ faster than the convergence time of TD(0). As far as we are aware, this is the first result rigorously showing benefits from parallelism for temporal difference methods.
\end{abstract}

\section{Introduction}

Recent years have seen reinforcement learning used in a variety of multi-agent systems, for example, cooperative control \citep{wang2020cooperative}, traffic control \citep{kuyer2008multiagent, bazzan2009opportunities}, networked robotics \citep{yang2004multiagent, duan2016benchmarking}, and bidding and advertising \citep{jin2018real}. However, a rigorous understanding of how standard methods in reinforcement learning perform in a multi-agent setting with limited communication is only partially available. 

One of the most fundamental problems in reinforcement learning is policy evaluation, and one of the most basic policy evaluation algorithms is temporal difference (TD) learning, originally proposed in \citet{sutton1988learning}. TD learning works by updating a value function from differences in predictions over a succession of steps in the underlying Markov Decision Process (MDP). 

Developments in the field of multi-agent reinforcement learning (MARL) have led to an increased interest in decentralizing TD methods, which is the subject of this paper. We will consider two different MARL settings. These are described formally below, but, in brief, the ``global state'' setting considers a collection of agents interacting with an environment which takes actions depending on the actions of all the agents, which may have different rewards; and the ``local state'' setting, involves each agent having its own copy of the same MDP. In both settings, the goal is to find a policy maximizing the average of the discounted reward streams. 



\subsection{Related Literature}

\noindent {\bf Analysis of centralized TD algorithm:} A natural benchmark to compare the performance of distributed TD methods to is the performance of centralized TD methods. In the context of linear function approximation, these date to 
\citet{jaakkola1994convergence,tsitsiklis1997analysis}. Precise conditions for the asymptotic convergence result was first given in \citet{tsitsiklis1997analysis} by viewing TD as a stochastic approximation for solving a Bellman equation. Recently, there has been an increased interest in non-asymptotic convergence results, 
e.g., \citet{dalal2018finite0, lakshminarayanan2018linear,bhandari2018finite}. The state of the art results show that, under i.i.d samples, TD algorithm with linear function approximation converge as fast as $O(1/\sqrt{T})$ for value function with step-size $1/\sqrt{T}$ and converge as fast as $O(1/t)$ with step-size $O(1/t)$ \citep{bhandari2018finite}.

\noindent {\bf Analysis of distributed TD methods:} There has been several recent non-asymptotic analyses of distributed TD with linear function approximation. All of these were in the global state model informally described above. The first paper on the subject \citet{doan2019finite} proposed a distributed variant of the TD algorithm by combining consensus step and local TD updates for the update of each agent. For distributed TD algorithm with projection step and i.i.d samples, \citet{doan2019finite} provided a $O((\log T) / \sqrt{T})$ convergence rate for value function with step-size that scales as $1/\sqrt{T}$ and $O((\log t)/t)$ convergence rate for optimal value with step-size $O(1/t)$. 

However, the constants in these bounds scaled with the spectral gap of the matrix underlying inter-agent communications. It is mentioned in  \citet{doan2019finiteb} that these could be as bad as $O(N^3)$ in terms of the number of agents $N$. This is of the main issues we will try to address in this paper: we would like to derive bounds that either benefit from, or are not hurt by, the multi-agent nature of the system. That is, we would like bounds that either do not get worse with $N$ or are actually improved as $N$ becomes large. 

In \citet{sun2019finite}, the case of both i.i.d and Markov samples was studied using a Lyapunov approach.  It was shown that all local estimates converge linearly to a small neighborhood of the optimum with constant step-size. The size of the neighborhood scaled with the inverse of the underlying spectral gap of a matrix based on the pattern of inter-agent communications (and thus, indirectly, with the number of agents for many communication patterns); this was removed in \citet{wang2020decentralized} using a more sophisticated ``gradient-tracking'' approach, which involves communicating twice as much information per each step. 

Finally, we also mention \citet{shen2020asynchronous} although that work deals with actor-critic rather than temporal difference methods. It is shown there, up to a certain approximation error, it is possible to obtain a linear speedup for a distributed model of actor-critic with independent samples across agents; this is in the same spirit as what we are attempting to do in part of this work. 



\subsection{Our Contribution}

This paper provides a simple scheme for both the global state model studied in the previous literature and the local state model we introduce here under i.i.d. sampling. In contrast to previous papers which required communication between neighbors at every step of the underlying methods, our schemes requires only one global average computation at the very end. 

We show that we are able to replicate the standard bounds for TD(0) in the global state model, including convergence to the same limit $\theta^*$ without any dependence on the number of agents. In the local state model, we show that we can improve state-of-the-art convergence times for TD(0): in other words, there is a benefit from parallelism in the local state setting. In particular, to the extent that the variance of the temporal difference error enters the convergence bounds for the local state model, it can be divided by the number of agents $N$. 



\section{Preliminaries}

We begin by standardizing notation and providing standard background information on Markov Decision Processes and temporal difference methods.

\subsection{Markov Decision Processes}
A discounted reward MDP is described by a 5-tuple $(\mathcal{S},\mathcal{A},\mathcal{P},r,\gamma)$, where $\mathcal{S}=[n]=\{1,2,\cdots,n\}$ is a finite state space, $\mathcal{A}$ is a finite action space, $\mathcal{P}(s'|s,a):\mathcal{S} \times \mathcal{A} \times \mathcal{S} \rightarrow [0,1]$ is transition probability from $s$ to $s'$ determined by $a$, $r(s,a,s'): \mathcal{S} \times \mathcal{A} \times \mathcal{S} \rightarrow \mathbb{R}$ are deterministic rewards and $\gamma \in (0,1)$ is the discount factor. 

Let $\mu$ denote a fixed policy that maps a state $s \in \mathcal{S}$ to a probability distribution $\mu(\cdot|s)$ over the action space $\mathcal{A}$, so that $\sum_{a \in \mathcal{A}} \mu(a|s) =1$. For such a fixed policy $\mu$, define the instantaneous reward vector $R^{\mu} : \mathcal{S} \rightarrow \mathbb{R}$ as $$R^{\mu}(s)=\sum_{s' \in \mathcal{S}} \sum_{a \in \mathcal{A}} \mu(a|s)\mathcal{P}(s'|s,a)r(s,a,s').$$ Fixing the policy $\mu$ induces a probability transition matrix between states: $$P^{\mu}(s,s')=\sum_{a \in \mathcal{A}} \mu(a|s)\mathcal{P}(s'|s,a).$$ We will use $r_{t}=r(s_t,a_t,s_{t+1})$ to denote the instantaneous reward at time $t$, where $s_t$, $a_t$ are the state and action taken at step $t$. The value function of $\mu$, denoted by $V^{\mu}: \mathcal{S} \rightarrow \mathbb{R}$ is defined as 
\begin{equation}
 V^{\mu}(s)=E_{\mu,s}\left[\sum_{t=0}^{\infty} \gamma^tr_{t}\right], \label{eq:valuefunction}
\end{equation}where $E_{\mu,s} \left[ \cdot \right]$ indicates that $s$ is the initial state and the actions are chosen according to the policy $\mu$. In the following, we will treat $V^{\mu}$ and $R^{\mu}$ as vectors in $\mathbb{R}^n$ and treat $P^{\mu}$ as a matrix in $\mathbb{R}^{n \times n}$. 

Next, we state a standard assumptions on the underlying Markov chain.
\begin{assumption} \label{ass:mc}
The Markov chain with transition matrix $P^{\mu}$ is irreducible and aperiodic. 
\end{assumption}
A consequence of Assumption \ref{ass:mc} is that there exists a unique stationary distribution $\pi = (\pi_1, \pi_2, \cdots, \pi_n)$, a row vector whose entries are positive and sum to $1$. This stationary distribution satisfies $\pi^T P^{\mu} =\pi^T$ and $\pi_{s'} = \lim_{t \rightarrow \infty} (P^{\mu})^t(s,s')$ for any two states $s,s' \in \mathcal{S}$. {\em Note that we use $\pi$ to denote the stationary distribution and $\mu$ to denote the policy.}


We next provide definitions of two norms that we will have occasion to use later. For a positive definite matrix $A \in \mathbb{R}^{n \times n}$, we define the inner product $\langle x, y \rangle_{A} = x^T A y$ and the associated norm $\|x\|_A = \sqrt{x^T A x}$ respectively. Since the numbers $\pi_{s}$ are positive for all $s \in \mathcal{S}$, then the diagonal matrix $D = {\rm diag}(\pi_1,\cdots,\pi_n) \in \mathbb{R}^{n \times n}$ is positive definite. Therefore, for any two vectors $V,V' \in \mathbb{R}^n$, we can also define an inner product as $$\left \langle V,V' \right \rangle _{D} = V^T D V'=\sum_{s \in \mathcal{S}}\pi_s V(s) V'(s),$$and the associated norm as 
\begin{equation}\label{eq:def_D}
    \|V\|_{D}^2=V^T D V = \sum_{s \in \mathcal{S}} \pi_s V(s)^2.
\end{equation}

Finally, we introduce the definition of Dirichlet seminorm, following the notation of \citet{ollivier2018approximate}: 
\begin{equation}\label{eq:def_Dir}
\|V\|_{{\rm Dir}}^2 = \frac{1}{2} \sum_{s,s' \in \mathcal{S} } \pi_s P^{\mu}(s,s') (V(s')-V(s))^2.
\end{equation}
Note that Dirichlet seminorm depends both on the transition matrix $P^{\mu}$ and the stationary distribution $\pi$. 
Similarly, we introduce the $k$-step Dirichlet seminorm: 
 $$\|V\|_{{\rm Dir},k}^2 = \frac{1}{2} \sum_{s,s' \in \mathcal{S}} \pi_s (P^{\mu})^k(s,s')(V(s')-V(s))^2.$$

\subsection{Temporal Difference Learning}
Evaluating the value function $V^{\mu}$ of a policy can be computationally expensive when the number of states is very large. The classical TD algorithm  uses low dimensional approximation $V^{\mu}_{\theta}$. 
For brevity, we will omit the superscript $\mu$ throughout from now on.

We next introduce the update rule of the classical temporal difference method with linear function approximation $V_{\theta}$, a linear function of $\theta$: 
\begin{equation}\label{eq:V_theta}
 V_{\theta}(s)=\sum_{l=1}^K \theta_l \phi_l(s) \quad \forall s \in \mathcal{S},
\end{equation}
where $\phi_l = (\phi_l(1),\cdots,\phi_l(n))^T \in \mathbb{R}^n$ for $l \in [K]$ are $K$ given feature vectors. Together, all K feature vectors form a $n \times K$ matrix $\Phi = (\phi_1, \cdots,\phi_K)$. For $s \in \mathcal{S}$, let $\phi(s) = (\phi_1(s),\cdots,\phi_K(s))^T \in \mathbb{R}^K$ denote the $s$-th row of matrix $\Phi$, a vector that collects the features of state $s$. Then, Eq. (\ref{eq:V_theta}) can be written in a compact form $V_{\theta}(s) = \theta^T \phi(s)$. 

The TD(0) method maintains a parameter $\theta(t)$ which is updated at every step to improve the approximation. Supposing that we observe a sequence of states $\{s(t)\}_{t \in \mathbb{N}_0}$, then the classical TD($0)$ algorithm updates as:
\begin{align}
 \theta ({t+1})= \theta (t) + \alpha_t \delta(t) \phi(s(t)), \label{eq:lineartd} 
\end{align}where $\{\alpha_t \}_{t \in \mathbb{N}_0}$ is the sequence of step-sizes, and letting $s'(t)$ denote the next state after $s(t)$, the quantity $\delta (t)$ is the temporal difference error
\begin{equation}\label{eq:delta}
 \delta(t) = r (t) + \gamma \theta ^T(t) \phi(s' (t)) - \theta ^T(t) \phi(s (t)).
\end{equation}



A common assumption on feature vectors in the literature \citep{tsitsiklis1997analysis,bhandari2018finite} is that features are linearly independent and uniformly bounded, which is formally given next.
\begin{assumption}\label{ass:features}
The matrix $\Phi$ has full column rank, i.e., the feature vectors $\{ \phi_1, \ldots, \phi_K\}$ are linearly independent. Additionally, we have that $ \|\phi(s)\|_2^2 \leq 1$ for $s \in \mathcal{S}$.
\end{assumption}

Under Assumption \ref{ass:mc} and \ref{ass:features}, we introduce the steady-state feature covariance matrix $\Phi^T D \Phi$. That this is a positive definite matrix as an immediate consequence of Assumptions \ref{ass:mc} and \ref{ass:features}, and we let $\omega > 0$ be a lower bound on its smallest eigenvalue. 


We will use the fact, shown in \citet{tsitsiklis1997analysis}, that under Assumptions \ref{ass:mc}-\ref{ass:features} as well as an additional assumption on the decay of the step-sizes $\alpha_t$, the sequence of iterates $\{\theta_t\}$ generated by TD($0$) learning converges almost surely a vector satisfying a certain projected Bellman equation; we will use $\theta^*$ to refer to this vector.

\subsection{Two Models of Distributed Temporal Difference Methods}


We now introduce two distributed models, which we will refer to as the {\em global state} and {\em local state} models. The global state model was previously introduced in \citet{doan2019finite} and studied in \citet{wang2020decentralized, sun2019finite}. We are not aware of the local state model being considered in the previous literature.

\subsubsection{The Global State Model} The problem is characterized by the 6-tuple $(\mathcal{S},\mathcal{V}, \mathcal{A}_{\mathcal{V}} , \mathcal{P},r_{\mathcal{V}}, \gamma)$. Here, $\mathcal{S}$, $\mathcal{P}$, and $\gamma$ have the same definition as before; $\mathcal{V} = [N] = \{1, \ldots, N\}$ is the set of agents; $ \mathcal{A}_{\mathcal{V}} = \mathcal{A}_1 \times \cdots \times \mathcal{A}_N$ is the set of joint actions, where $\mathcal{A}_v$ is a set of actions only available to agent $v$; and $r_{\mathcal{V}} = \{r_v\}_{v \in {\mathcal{V}}}$ is a set of reward functions, where $r_v(s,a,s'): \mathcal{S} \times \mathcal{A} \times \mathcal{S} \rightarrow \mathbb{R}$ is the reward function of agent $v$.

The policy $\mu$ now takes the form $\mu(a|s) = \prod_{v=1}^N \mu_v(a_v|s)$, where $\mu_v(a_v|s)$ is the probability to select action $a_v \in \mathcal{A}_v$ when in state $s$. After the action $a(t) = \{a_{1}(t),\cdots,a_{N}(t)\}$, the system moves to a new state $s'(t)$ with probability $\mathcal{P}(s'(t)| s(t), a(t))$; then all the agents observe the new state and obtain local rewards $r_{v}(t) = r_v(s_v(t),a(t),s'_v(t))$. 


The value of a policy will depend on the average of the rewards obtained by the individual agents: 
\begin{equation}
 V(s) = E_{\mu,s} \left[ \sum_{t=0}^{\infty} \frac{\gamma^t}{N} \sum_{v \in \mathcal{V}} r_v(s_v(t),a(t),s'_v(t) \right], \quad i \in \mathcal{S}. \label{eq:polvalue}
\end{equation}

In distributed policy evaluation, the agents wish to cooperate to estimate the reward $V$. We could do this by applying TD($0$) to the reward average $(1/N) \sum_{v} r_v(t)$. However, this is not naturally distributed since the reward average depends on what happens at every node in the network. Nevertheless, let  $\theta_{\rm gl}^*$ be the limit point of this  method; 
our goal is to converge to this $\theta_{ \rm gl}^*$  with a fully distributed method. 

A natural way to do this is to distribute the TD($0$) method as we do in Algorithm \ref{algo:doan}. Informally, Algorithm \ref{algo:doan} starts from an arbitrary parameter vector $\theta_v(0)$. At each iteration $t \in \mathbb{N}_0$, the agents take actions, the system moves to a new random state $s'(t)$ based on these actions, and agent $v$ observes the a tuple $(s(t),s'(t),r_v(t))$. It then executes the TD($0$) algorithm on this tuple. This is done for $T$ steps, and then the system ``outputs'' the average across the network of the running averages of the iterates maintained by each individual node, and the average across the network of the latest estimates. Communication among workers is required {\bf only} in the final step to average their parameters. 

\begin{algorithm}[ht]
\caption{TD($0$) with Global State} \label{algo:doan}
\begin{algorithmic}[1]
\STATE For {$v \in \mathcal{V}$}, initialize $\theta_v(0)$, $s(0)$ 
\FOR {$t=0$ to $T-1$} 
 \FOR{$v \in \mathcal{V}$}
 \STATE Observe a tuple $(s(t),s'(t),r_v(t))$.
 \STATE Compute temporal difference: $$\delta_v(t) = r_v(t) - \left(\phi(s(t))- \gamma \phi(s'(t))\right)^T \theta_v(t).$$
 \STATE Execute local TD update: 
 \begin{equation} \label{eq:DTD}
 {\theta}_v({t+1}) = \theta_v(t) + \alpha_t \delta_v(t)\phi(s(t)).
 \end{equation}
 \STATE Update running average: 
 \[ \hat{\theta}_v(t+1) = \left( 1 - \frac{1}{t+2} \right) \hat{\theta}_v(t) + \frac{1}{t+2} \theta_v(t+1). \]
 \ENDFOR
\ENDFOR
\STATE Return $\hat{\theta}(T)$ and $\bar{\theta}(T)$: \label{step:last}
\begin{equation*}
 \hat{\theta}(T) = \frac{1}{N} \sum_{v \in \mathcal{V}} \hat{\theta}_v(T), \quad \bar{\theta}(T)= \frac{1}{N} \sum_{v \in \mathcal{V}} {\theta}_v(T).
\end{equation*}
\end{algorithmic}
\end{algorithm}

{\bf Message complexity:} Under the assumption that the nodes are connected to a server, computing the average in step 10 takes a single round of communication with a server. In the nearest-neighbor model where the nodes are connected over an undirected graph and nodes know the total number of nodes $N$, it is possible to find an $\epsilon$-approximation of the average in $O(N \log (1/\epsilon))$ time using the algorithm from \citet{olshevsky2017linear}. If such knowledge is not available, and the communication graph is further time-varying, it is possible to do the same in $O(N^2 \log (1/\epsilon)$ using the algorithm from \citet{nedic2009distributed}. As we will later discuss, it suffices to choose $\epsilon$ proportional to a power of $1/T$, so that the message complexity of step 10 in the fully distributed setting is at most $O(\log T)$. For large enough $T$, this is an exponential improvement over the previous papers \cite{doan2019finite, doan2020finite, sun2019finite, wang2020decentralized} which required communication at every step and thus needed $T$ communications. 

\subsubsection{The Local State Model} 



We next consider a model where each agent has its own independently evolving copy of the same MDP. Although this seems quite different form the global state model, the two models can be treated with a very similar analysis. 

More formally, each agent has the same 6-tuple $(\mathcal{S},\mathcal{V},\mathcal{A}, \mathcal{P}, r, \gamma)$; at time $t$, agent $v$ will be in a state $s_v(t)$; it will apply action $a_{v}(t) \in \mathcal{A}$ with probability $\mu( a_v(t)|s_v(t))$; then agent $v$ moves to state $s'_v(t)$ with probability $\mathcal{P}(s'_v(t)| s_v(t), a_v(t))$, {with the transitions of all agents being independent of each other}; finally agent $v$ gets a reward $r_{v}(t) = r(s_v(t),a_v(t),s'_v(t))$. Note that, although the rewards obtained by different agents can be different, the reward function $r(s,a,s')$ is identical across agents. 


We define $\theta_{\rm lc}^*$ to be the fixed point of TD($0$) on the MDP $(\mathcal{S},\mathcal{V},\mathcal{A}, \mathcal{P}, r, \gamma)$.
Naturally, each agent can easily compute $\theta_{\rm lc}^*$ by simply ignoring all the other agents. However, this ignores the possibility that agents can benefit from communication with each other.

We propose a distributed TD method in this setting as Algorithm \ref{algo:our}. It is very similar to to Algorithm \ref{algo:doan}: each agent runs TD($0$) locally at each agent, and, at the end, the agents just average the results. The message complexity of this method is the same as the message complexity of Algorithm 1: one communication with a server if a server is assumed to be available and $O(\log T)$ communications in the fully distributed setting. 

\begin{algorithm}[ht]
\caption{TD($0$) with Local State} \label{algo:our}
\begin{algorithmic}[1]
\STATE For {$v \in \mathcal{V}$}, initialize $\theta_v(0)$, $s_v(0)$ 
\FOR {$t=0$ to $T-1$} 
 \FOR{$v \in \mathcal{V}$}
 \STATE Observe a tuple $(s_v(t),s_v'(t),r_v(t))$.
 \STATE Compute temporal difference: \begin{equation}\label{eq:delta_lc}
 \delta_v(t) = r_v(t) - \left(\phi(s_v(t))- \gamma \phi(s_v'(t))\right)^T \theta_v(t).
 \end{equation}
 \STATE Execute local TD update: 
 \begin{equation} \label{eq:DTD2}
 {\theta}_v({t+1}) = \theta_v(t) + \alpha_t \delta_v(t)\phi_v(s_v(t)).
 \end{equation}
 \STATE Update running average: 
 \[ \hat{\theta}_v(t+1) = \left( 1 - \frac{1}{t+2} \right) \hat{\theta}_v(t) + \frac{1}{t+2} \theta_v(t+1). \]
 \ENDFOR
\ENDFOR
\STATE Return $\hat{\theta}(T)$ and $\bar{\theta}(T)$:
\begin{equation*}
 \hat{\theta}(T) = \frac{1}{N} \sum_{v \in \mathcal{V}} \hat{\theta}_v(T), \quad \bar{\theta}(T)= \frac{1}{N} \sum_{v \in \mathcal{V}} {\theta}_v(T).
\end{equation*}
\end{algorithmic}
\end{algorithm}


\subsection{Why Two Models?} The global state model was introduced in the previous literature \cite{doan2019finite, wang2020decentralized, sun2019finite}. It is a natural starting point for multi-agent RL where the state of the system depends on what all the agents do. 

However, a shortcoming of the global state model is that it cannot be used to parallelize temporal difference learning. Indeed, consider the global state model with the proviso that all rewards except the rewards of the first agent are zero, and only the action of the first agent affects the transition of the global state. Then the global model reduces to the regular policy evaluation problem for one agent. As a consequence, the best we can hope to achieve in the global state model is to recover the guarantees for classical TD learning. 

It may be objected that it is somewhat artificial to only have the first agent make decisions that matter, but it is easy to come up with more involved examples with the same property (e.g., when all rewards are the same and the system evolves according to the action chosen by the most agents). 

By contrast, in the local state model, there is the possibility of doing better than the regular TD($0$) because the agents are ``collectively'' observing $n$ tuples $(s_v(t), s_v'(t), r_v(t))$ per step, whereas classical TD($0$) only gets to observe a single tuple (though it must be stressed that each agent in the local state model only observes its own tuple). Since these tuples are generated independently, to the extent that the variance of the temporal difference error affects the performance of TD($0$), there is the possibility of achieving performance that is a factor of $N$ times better.
\section{Convergence Analyses of Our Methods} 

We next describe the main results of this paper, which are convergence analyses of Algorithms \ref{algo:doan} and \ref{algo:our} 
under the assumption that the tuples are i.i.d. In the literature, the i.i.d model is sometimes referred to as having a ``generator'' for the MDP and is a more restrictive assumption compared to assuming that the state evolves as a Markov process with a fixed starting state. Nevertheless, this is a standard assumption under which many TD and Q-learning methods are analyzed (e.g., \citep{sutton2008convergent,dalal2018finite0,dalal2018finite,lakshminarayanan2018linear,doan2019finite,chen2019finite,kumar2019sample}).


\subsection{Distributed TD(0) with Global State Model}



Before stating our result, we need to introduce notation for the variance of the temporal difference error. Let $\bar{r}(t) = (1/|N|) \sum_{v} r_v(t)$ be the average of the instantaneous rewards received by the agents at time $t$. Then we define $$\bar{\sigma}^2 = E\left[ \left(\bar{r}(t) - \left(\phi(s(t))- \gamma \phi(s'(t))\right)^T \theta_{\rm gl}^*\right) ^2\right].$$ Recall that the expectation is taken with respect to the distribution that generates the state $s$ with probability $\pi_s$, then actions $(a_1, \ldots, a_n)$ from the policy, and the next state $s'(t)$ from the transition of the MDP. 
 
Our first main result bounds the performance of distributed TD(0) with global state in terms of $\bar{\sigma}^2$ as well as the initial distance to the optimal solution. 

\begin{theorem} \label{thm:1}
Suppose Assumptions \ref{ass:mc}-\ref{ass:features} hold. Suppose further that $\{ {\theta}_v(t)\}_{v \in \mathcal{V}}$ and $\{ \hat{\theta}_v(t)\}_{v \in \mathcal{V}}$ are generated by Algorithm \ref{algo:doan} in the global state model where the state $s(t)$ is sampled i.i.d according to the stationary distribution $\pi$. Then,

(a) For any constant step-size sequence $\alpha_0 = \cdots = \alpha_T= \alpha \leq (1- \gamma)/8$,
\begin{equation*}
    E\left[\left\|\bar{\theta}({T}) - \theta_{\rm gl}^*\right\|_2^2 \right] 
 \leq  e^{-\alpha (1-\gamma) \omega T} E\left[ \left\|\bar{\theta}({0}) - \theta_{\rm gl}^*\right\|_2^2 \right] + \frac{2\alpha \bar{\sigma}^2}{ (1-\gamma) \omega}.
\end{equation*}

(b) For any $T \geq \frac{64}{ (1-\gamma)^2}$ and constant step-size sequence $\alpha_0 = \cdots = \alpha_T =\frac{1}{\sqrt{T}}$, 
\begin{equation*}
      (1-\gamma) E\left[ \left\|V_{\theta_{\rm gl}^*} - V_{\hat{\theta}(T)} \right\|_{D}^2\right] + 2 \gamma E \left[\left\|V_{\theta_{\rm gl}^*} - V_{\hat{\theta}(T)} \right\|_{{\rm Dir}}^2 \right]  \leq  \frac{E\left[ \left\|\bar{\theta}({0}) - \theta_{\rm gl}^*\right\|_2^2 \right] + \bar{\sigma}^2 }{\sqrt{T}}.
\end{equation*}

(c) For a decaying step-size sequence $\alpha_t = \frac{\alpha}{t+\tau}$ with $\alpha = \frac{2}{(1-\gamma)\omega}$ and $\tau = \frac{16}{(1-\gamma)^2\omega}$, we have, 
\begin{equation*}
 E\left[\left\|\bar{\theta}({T}) - \theta_{\rm gl}^*\right\|_2^2 \right] \leq \frac{\zeta}{T+\tau},
\end{equation*}where $\zeta = \max\left\{ {2 \alpha^2 \bar{\sigma}^2}, \tau \left\|\bar{\theta}({0}) - \theta_{\rm gl}^*\right\|_2^2 \right\}$.
\end{theorem}

A formal proof can be found in the supplementary material. 


The first takeaway here is that the convergence rates to $\theta_{\rm gl}^*$ in parts (b) and (c)  are {\em  exactly the same as the existing convergence times for regular TD(0)}. Indeed, the state-of-the-art finite-time convergence analysis for TD(0) was given in the paper \cite{bhandari2018finite}, and the bounds given there are exactly the same as the ones in the above theorem, up to constant factors\footnote{Actually, the results here are slightly stronger compared to the results in \cite{bhandari2018finite}. The difference is that we give bounds on the quantity $(1-\gamma) E\left[ \left\|V_{\theta^*} - V_{\hat{\theta}(T)} \right\|_{D}^2\right] + 2 \gamma E \left[\left\|V_{\theta^*} - V_{\hat{\theta}(T)} \right\|_{{\rm Dir}}^2 \right]$, whereas the bounds of \cite{bhandari2018finite}, after some rearrangement, give the same upper bound on just $(1-\gamma) E\left[ \left\|V_{\theta_{\rm gl}^*} - V_{\hat{\theta}(T)} \right\|_{D}^2\right]$.}. As discussed earlier, the best one can hope for is to replicate the bounds for regular TD(0) as Algorithm \ref{algo:doan} contains the usual TD(0) as a special case. 

Similarly part (a) shows convergence of the distributed method to an $O(\alpha \bar{\sigma}^2/(\omega (1-\gamma))$ neighborhood of the optimal solution (measured in terms of the performance measure on the left-hand side). This is also equivalent to the asymptotic performance in their state-of-the-art analysis of regular TD(0) with fixed step-size from \citet{bhandari2018finite}.

\noindent {\bf Message complexity of the final consensus step:} For simplicity, we have given Theorem \ref{thm:1} under the assumption that the final averages $\hat{\theta}(T), \bar{\theta}(T)$ are computed exactly.  We now come back to the question of how many inter-neighbor communication steps are needed to implement step 10 (and preserve our theoretical guarantees) when only nearest-neighbor communications in a graph are allowed. 

It is immediate that all the quantities we bound in Theorem \ref{thm:1} (i.e., the left-hand sides of all the equations) are Lipschitz in a neighborhood of $\theta_{gl}^*$. Consequently,  to preserve a constant error in part (a), or $1/\sqrt{T}$ and $1/T$ errors in parts (b) and (c), it suffices to average with an error that is a small enough constant in part (a), and a small enough multiple of $1/\sqrt{T}$ and $1/T$ respectively in parts (b) and (c). 

As discussed earlier, to obtain an $\epsilon$-approximate average using state of the art distributed ``average consensus'' methods takes $O(\log (1/\epsilon))$ steps, where the constant in the $O(\cdot)$ notation will depend on $N$ or the spectral gap as well as assumptions we make about the graph. Thus we need to take $ \epsilon = O(1/\sqrt{T}), \epsilon=O(1/T)$ in cases (b),(c). This means  we will need to run an average consensus method for $O(\log T)$ communications to approximately implement step 10 of Algorithm 1 after running the previous loop for $T$ steps. Thus the final message complexity is $O( \log T)$.

\noindent {\bf Comparison to earlier work:} A similar algorithm was analyzed in \citet{doan2019finite}. The difference is that the agents were assumed to be connected by a (possibly time-varying) sequence of graphs; each agent communicated with neighbors at each step. The bound derived in \citet{doan2019finite} for step-size that scales as $1/\sqrt{k}$ with iteration $k$ was of the form \begin{small}
\[ (1-\gamma) E\left[ \left\|V_{\theta_{\rm gl}^*} - V_{\hat{\theta}(T)} \right\|_{D}^2\right] \leq O \left( \frac{||\bar{\theta}(0) - \theta_{\rm gl}^*||_D^2 + \left( \frac{1}{1-\lambda} \right) \ln T}{\sqrt{T}} \right), 
\] \end{small}where $\lambda$ was related to the spectral gap of the communication graphs; crucially, it was remarked in \citet{doan2019finiteb} that in the worst case, $1/(1-\lambda)$ was as large as $N^3$ on a fixed network of $N$ nodes. 

In other words, the bounds of \citet{doan2019finite} allowed for the possibility that distributed TD(0) performs worse than regular TD(0) in this setting. It might have been natural to guess that something like this is inevitable due to the multi-agent nature of the system. Our results show this is not the case. Not only can we match the performance of regular TD(0), we do not even need to communicate with neighbors except to average the estimates at the very end. 

We remark that \citet{doan2019finite} also considered the fixed step-size, and the follow-up paper \cite{doan2020finite} also considered step-sizes that scale with $1/t$. The comparison of parts (a) and (c) with these results is similar, as is the comparison of our results with the analysis of \citet{sun2019finite}: relative to all these papers, we both remove the scaling with the number of agents or with the eigengap of the underlying graph, while requiring no communication except for average computation at very end. 

We next compare this theorem with the results of \citet{wang2020decentralized}. That paper considered a more sophisticated "gradient tracking" scheme, where multiple quantities are shared among neighbors in an underlying graph at every step. Only the fixed step-size case was analyzed, and it was shown that the system converges to a neighborhood of the optimal solution whose size does not depend on the  number of agents or the spectral gap of a communication matrix. 

The above theorem improves on \citet{wang2020decentralized} in several ways. Besides significantly saving on communication by not requiring communication any communication until the very end, we give a clean expression for the final error: the right-hand side of Theorem \ref{thm:1}(a) is $\alpha \bar{\sigma}^2/(\omega(1-\gamma))$ in the limit as $T \rightarrow \infty$. 
Most importantly, in Theorem \ref{thm:1}(b) and Theorem \ref{thm:1}(c) we show convergence to the optimal solution itself without any dependence $n$ or a spectral gap (rather than only a neighborhood of it).

\subsection{Distributed TD(0) with Local State Model}


We now turn to the analysis of the local state model, beginning with  some notation. Recall that, for regular TD(0) in the global state model, convergence analysis will scale both with the distance to the initial solution, and with the variance $\bar{\sigma}^2$ of the temporal difference error with average reward. For the local model, the variance is identical to the variance defined in the centralized model: 
\begin{equation*}
 {\sigma}^2 = {\rm E} \left[\left( {r}(s,a,s') - \left(\phi(s)- \gamma \phi(s')\right)^T \theta_{\rm lc}^* \right)^2\right].
\end{equation*} As before,  the expectation is taken with respect to the distribution that generates the state $s$ with probability $\pi_s$, then actions $(a_1, \ldots, a_n)$ from the policy, and the next state $s'(t)$ from the transition of the MDP.

In the multi-agent case, we need some notion of the initial distance to the optimal solution;  we simply take the maximum over all the agents to define: \begin{equation*}
 \hat{R}_0 = \max_{v \in \mathcal{V}} E\left[ \left\|{\theta}_v({0}) - \theta_{\rm lc}^*\right\|_2^2 \right].
\end{equation*}

In the case where all agents start with the same initial condition, this reduces to the same quantities as we had before, i.e., $R_0 = ||\bar{\theta}(0) - \theta_{\rm lc}^*||_2^2$.




The following theorem is our second main result. 

\begin{theorem} \label{thm:main}
Suppose Assumptions \ref{ass:mc}-\ref{ass:features} hold. Suppose further that $\{ {\theta}_v(t)\}_{v \in \mathcal{V}}$ and $\{ \hat{\theta}_v(t)\}_{v \in \mathcal{V}}$ are generated by Algorithm \ref{algo:our} in the local state model under i.i.d sampling. Then,

(a) For any constant step-size sequence $\alpha_0 = \cdots = \alpha_T= \alpha \leq (1- \gamma)/8$, we have
\begin{equation*}
    E\left[(1-\gamma) \left\|V_{\theta_{\rm lc}^*} - V_{\hat{\theta}(T)} \right\|_{D}^2 + \gamma \left\|V_{\theta_{\rm lc}^*} - V_{\hat{\theta}(T)} \right\|_{{\rm Dir}}^2 \right] 
 \leq \frac{{\br 1}}{ {\br T }} \left(\frac{{\br 1}}{{\br 2\alpha}} {\br E}\left[ {\br ||\bar{\theta}({0}) - \theta_{\rm lc}^*||_2^2 } \right] +\frac{{\br 4\hat{R}_0}}{{\br 1-\gamma}} \right)+ \frac{ \alpha {{\bl \sigma}}^2}{{\bl N}} + \frac{{\br 8 \alpha^2 {\sigma}^2}}{{\br 1-\gamma}}.
\end{equation*}

(b) For any $T \geq \frac{64}{ (1-\gamma)^2}$ and constant step-size sequence $\alpha_0 = \cdots = \alpha_T =\frac{1}{\sqrt{T}}$, we have
\begin{equation*}
    E\left[(1-\gamma) \left\|V_{\theta_{\rm lc}^*} - V_{\hat{\theta}(T)} \right\|_{D}^2 + \gamma \left\|V_{\theta_{\rm lc}^*} - V_{\hat{\theta}(T)} \right\|_{{\rm Dir}}^2 \right] 
 \leq  \frac{1}{2 \sqrt{T} }\left( E\left[ \left\|\bar{\theta}({0}) - \theta_{\rm lc}^*\right\|_2^2 \right]+\frac{ 2 {{\bl \sigma}}^2}{{\bl N}} \right) +\frac{{\br 1}}{{\br T}} \left(\frac{{\br 4\hat{R}_0+ 8{\sigma}^2} }{{\br 1-\gamma}} \right).
\end{equation*}

(c) For the decaying step-size sequence $\alpha_t = \frac{\alpha}{t+\tau}$ with $\alpha = \frac{2}{(1-\gamma)\omega}$ and $\tau = \frac{16}{(1-\gamma)^2\omega}$. Then, 
\begin{equation*}
    E \left[ \left\|\bar{\theta}({t+1}) - \theta_{\rm lc}^*\right\|_2^2 \right] \leq \frac{2 \alpha ^2 {{\bl \sigma}}^2/{\bl N}}{t+\tau} + \frac{{\br 8 \alpha^2 \hat{\zeta} }}{{\br (t+\tau)^2}} + \frac{{\br (\tau-1)^4E \left[ \left\|\bar{\theta}({0}) - \theta_{\rm lc}^*\right\|_2^2 \right]} }{{\br (t+\tau)^4}},
\end{equation*}where $\hat{\zeta} = \max\left\{ {2 \alpha^2 {\sigma}^2}, \tau \hat{R}_0 \right\}$.
\end{theorem}

The proof of Theorem \ref{thm:main} is given in the supplementary material. 

To parse Theorem \ref{thm:main}, note that all the terms in brown are ``negligible'' in a limiting sense. Indeed, in part (a), the first term scales as $O(1/T)$ and consequently goes to zero as $T \rightarrow \infty$ (whereas the remaining terms do not). In parts (b) and (c), the terms in brown go to zero at an asymptotically faster rate compared to the dominant term (i.e., as $1/T$ vs the dominant $1/\sqrt{T}$ term in part(b) and as $1/t^2, 1/t^4$ compared to the dominant $1/t$ in part (c)). Finally, the last term in part (a) scales as $O(\alpha^2)$ and will be negligible compared to the term preceding it, which scales as $O(\alpha)$, when $\alpha$ is small. 

Moreover, among the non-negligible terms, whenever ${\sigma}^2$ appears, it is divided by $N$; this is highlighted in blue. 

To summarize, parts (b) and (c) show that, when the number of iterations is large enough, we can divide the variance term by $N$ as a consequence of the parallelism among $N$ agents. Part (a) shows that, when the number of iterations is large enough and the step-size is small enough, the size of the final error will be divided by $N$. 

Note that, in part (c), the result of this is a factor of $N$ speed up of the entire convergence time (when $T$ is large enough). In part (a), this results in a factor of $N$ shrinking of the asymptotic error (when the step-size $\alpha$ is small enough). In part (b), however, this only shrinks the ``variance term'' by a factor of $N$; the term depending on the initial condition is unaffected. The explanation for this is that in parts (a) and (c), the variance of the temporal difference error dominates the convergence rate, while in part (b) this is not the case. 

As far as we are aware, these results constitute the first example where parallelism was shown to help for distributed temporal difference learning. They also justify the introduction of the local state model in this paper: indeed, even if there is nothing multi-agent about the underlying problem, one might still choose to distribute the MDP among agents (which could be nodes in a computer cluster) in order to speed-up computation as guaranteed by this theorem.

\section{Numerical Experiments}

In this section, we perform some experiments to verify the conclusions of our theorems and compare Algorithm \ref{algo:our} with earlier work from \citet{doan2019finite} and \citet{wang2020decentralized}. Our experiments are performed on classic control problems from OpenAI gym and Gridworld; details are given in the supplementary materials. 

We focus on the case of constant step-size, since this both matches what is usually done in practice (where a fixed but small step-size is typically picked) and results in faster convergence fitting within our limited computation budget. Normally, a choice of step-size of $\alpha_t = \alpha$ results in an error of $O(\alpha)$ around the optimal solution. But according to Theorem \ref{thm:main}(a), choosing $N = \frac{1-\gamma}{8 \alpha}$ will result in a final error that is a much smaller $O(\alpha^2)$. 

In other words, if we were to plot the inverse of the variance of the final answer, we should see it grows linearly in one agent, and quadratically with $N$ agents chosen as above. This is exactly what Figure \ref{fig:var} below shows, plotting as a function of $\alpha^{-1}$ to make the quadratic vs linear distinction happen when $\alpha \rightarrow +\infty$, thus making it more visible. 

Note that the step-size $\alpha$ can be thought of trading off between the quality of the final solution, which is $O(\alpha)$, and the convergence time (which scales with $\alpha^{-1}$). These graphs show that we can use parallelism to get a much more accurate solution ($O(\alpha^2)$ error instead of $O(\alpha)$).



\begin{figure}[ht]
\centering  
\subfigure[Grid World]{
\label{Fig.sub.1}
\includegraphics[width=0.45\textwidth]{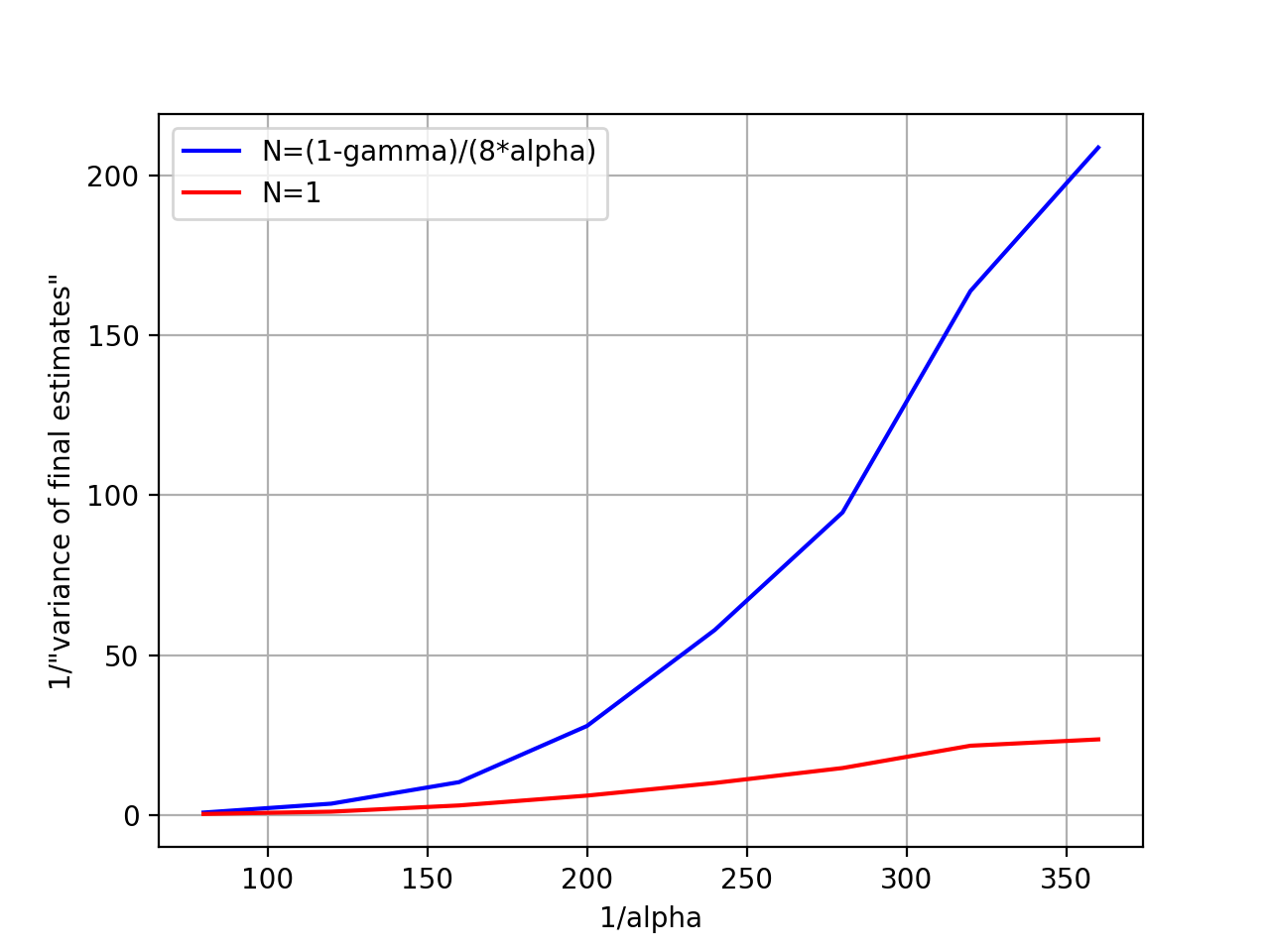}}
\subfigure[Acrobot-v1]{
\label{Fig.sub.2}
\includegraphics[width=0.45\textwidth]{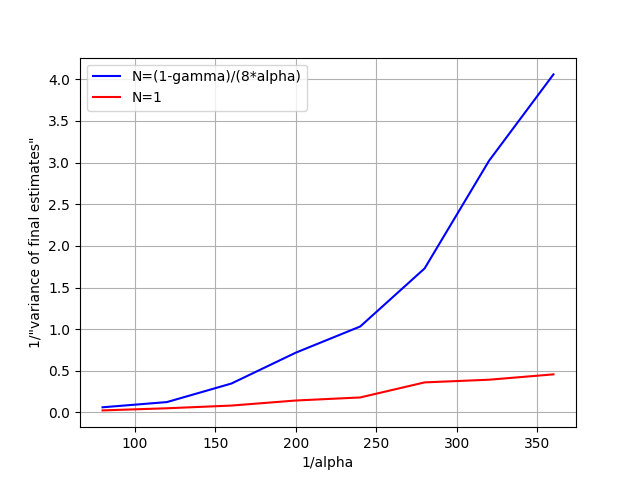}}
\subfigure[CartPole-v1]{
\label{Fig.sub.3}
\includegraphics[width=0.45\textwidth]{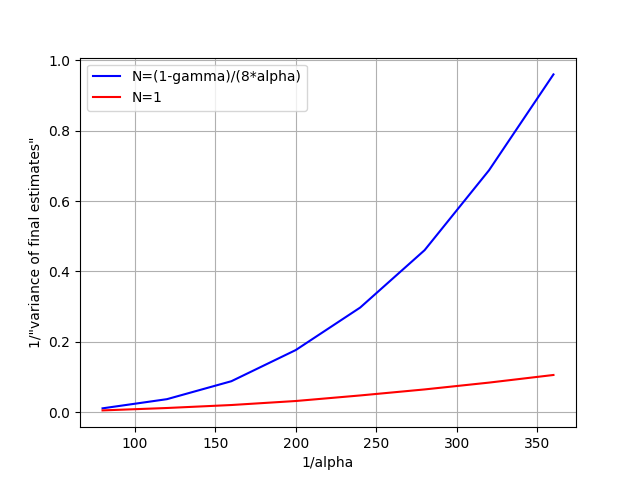}}
\subfigure[MountainCar-v1]{
\label{Fig.sub.4}
\includegraphics[width=0.45\textwidth]{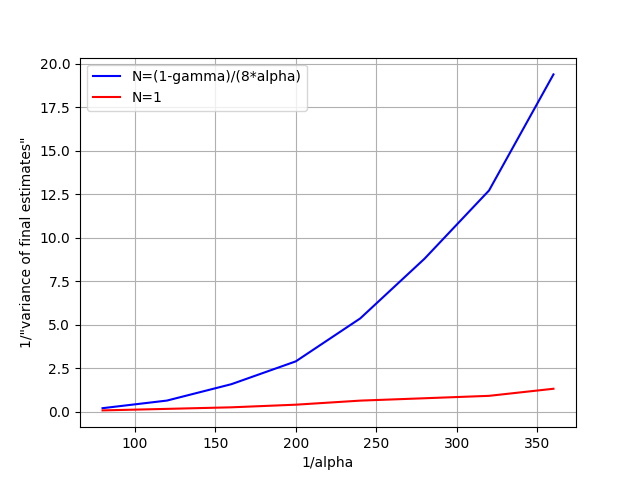}}
\subfigure[MountainCarContinous-v0]{
\label{Fig.sub.5}
\includegraphics[width=0.45\textwidth]{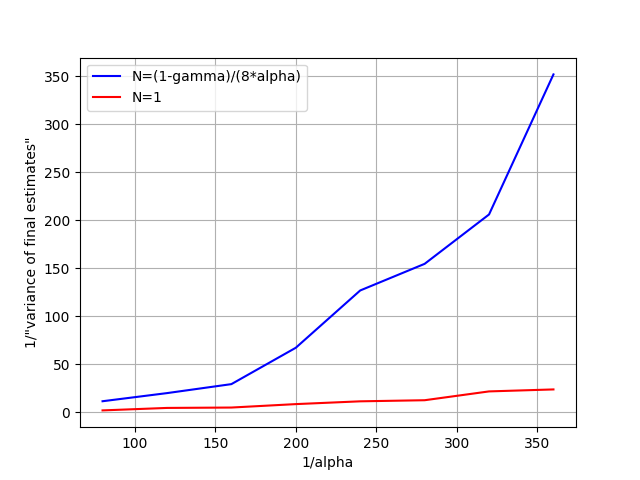}}
\subfigure[Pendulum-v1]{
\label{Fig.sub.6}
\includegraphics[width=0.45\textwidth]{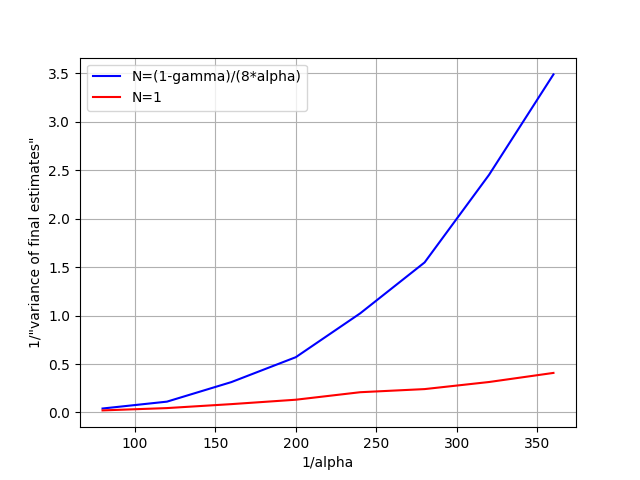}}
\caption{We plot the inverse of the empirical variance of the final solution vs $\alpha^{-1}$ on the x-axis. The policy takes uniformly random actions, and the function approximation uses tile coding. Details are available in the supplementary information.}
\label{fig:var}
\end{figure}



Our second set of simulations compare Algorithm \ref{algo:our} with earlier distributed TD methods stated in \citet{doan2019finite}, \citet{sun2019finite} and \citet{wang2020decentralized} in terms of TD error. The distributed TD methods of \citet{doan2019finite} and \citet{sun2019finite} are the same expect that \citet{doan2019finite} has an additional projection step (these two methods can be viewed as the same if one chooses a large enough set for the projection step). Again, we consider constant step size. The number of agents $N = 100$. The communication graph among agents is generated by the Erdos–Renyi model, which is connected. {\bf Recall that our method only uses one run of average consensus at the end, whereas the other methods require a communication at every step}. The graphs for our method show the TD error at each iteration if we stopped the method and run the average consensus to average the estimates across the network.  Figure \ref{fig:compare_six} shows that the TD errors of Algorithm \ref{algo:our} perform essentially identically to the other methods in spite of the reduced communication. 

\begin{figure}[ht]
\centering  
\subfigure[Grid World]{
\label{Fig.sub.1'}
\includegraphics[width=0.45\textwidth]{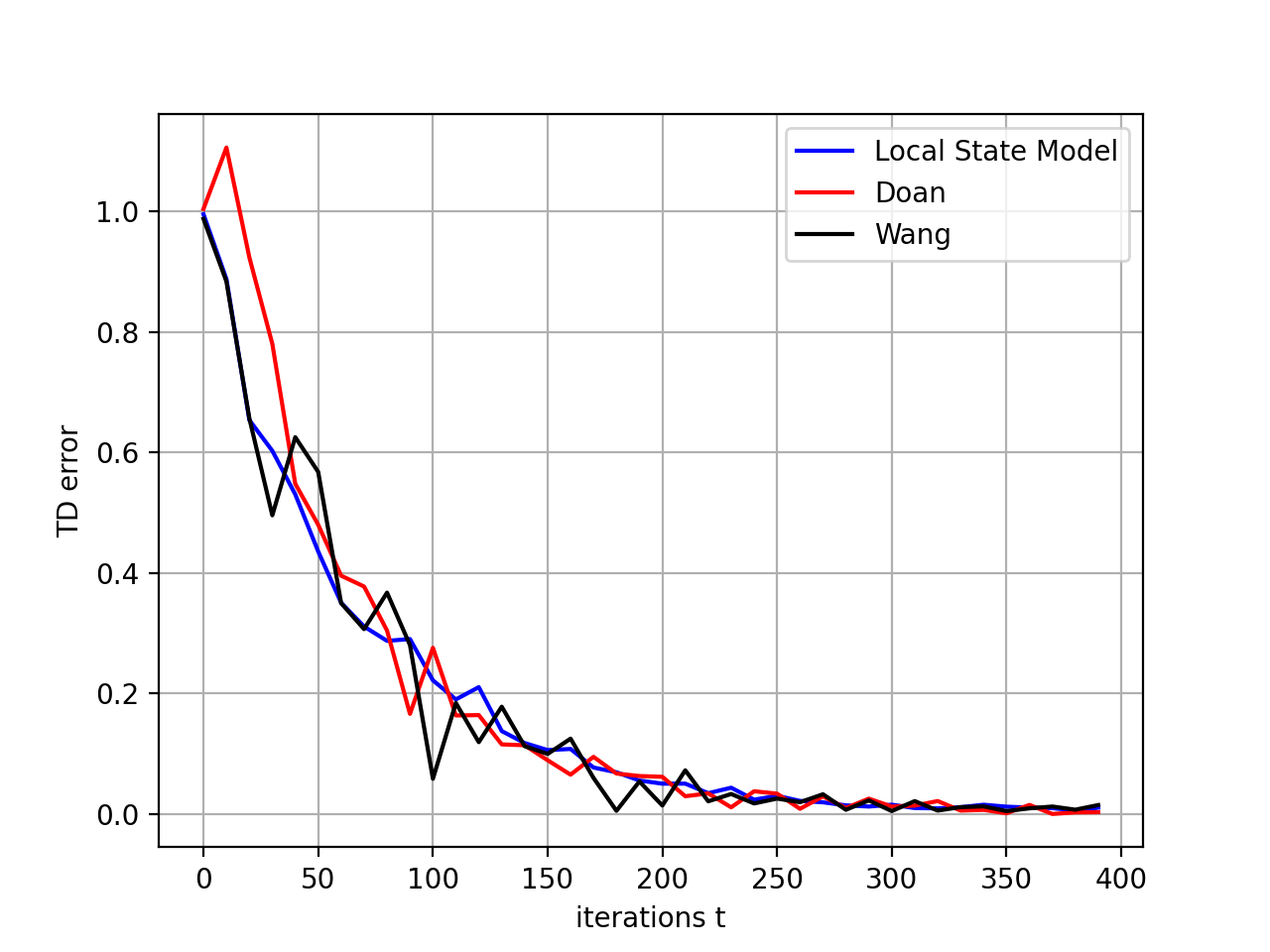}}
\subfigure[Acrobot-v1]{
\label{Fig.sub.2'}
\includegraphics[width=0.45\textwidth]{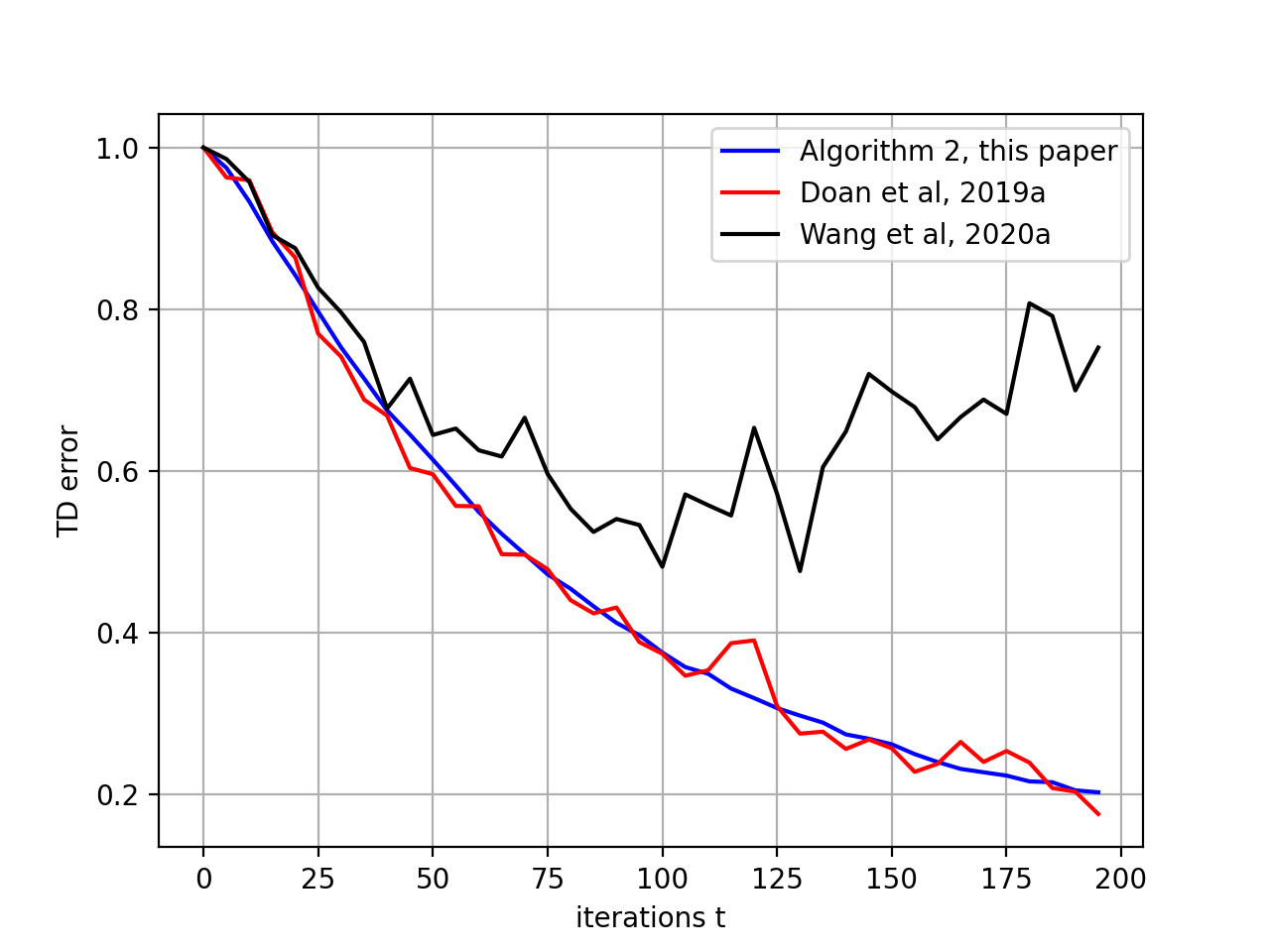}}
\subfigure[CartPole-v1]{
\label{Fig.sub.3'}
\includegraphics[width=0.45\textwidth]{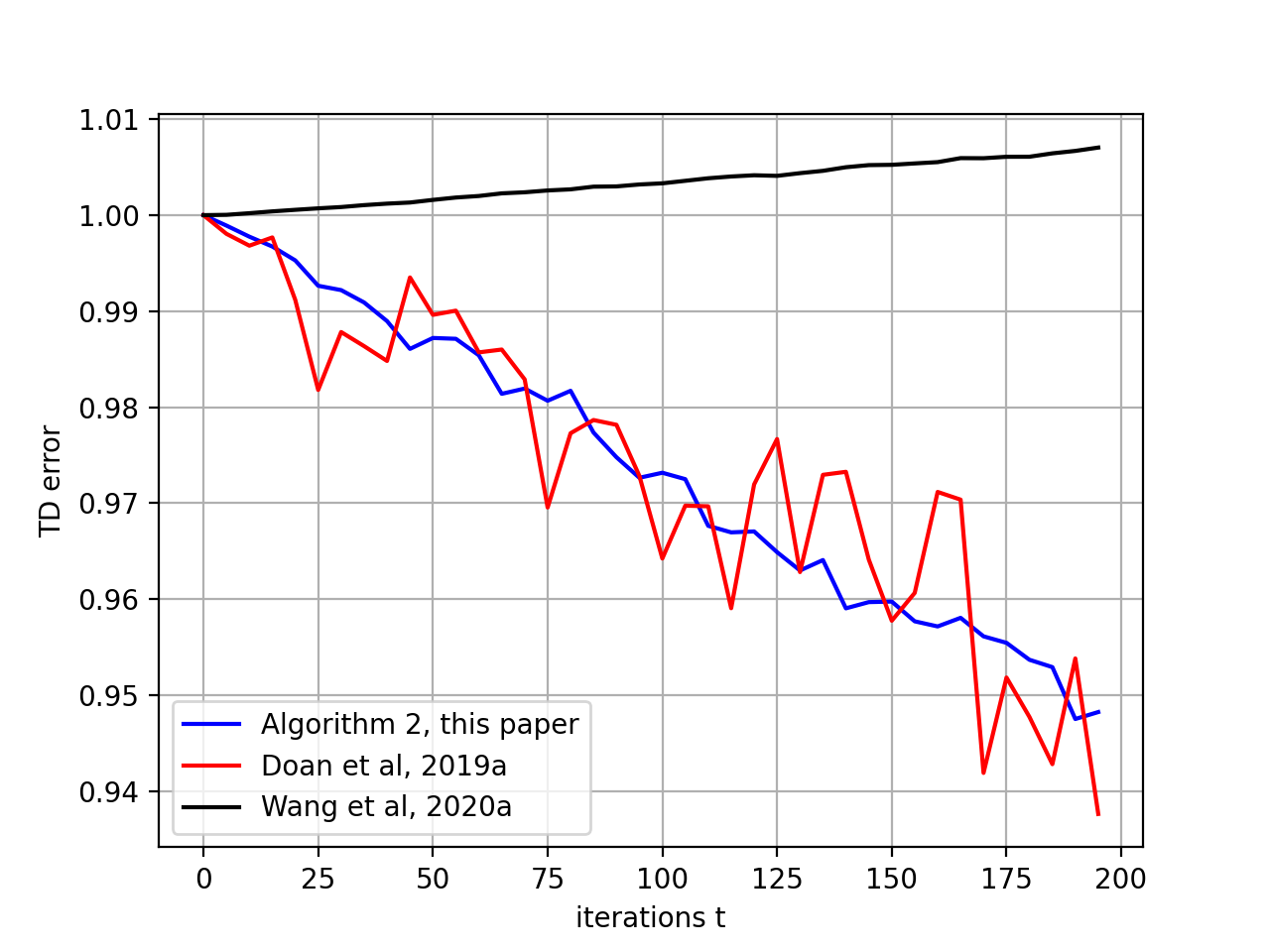}}
\subfigure[MountainCar-v1]{
\label{Fig.sub.4'}
\includegraphics[width=0.45\textwidth]{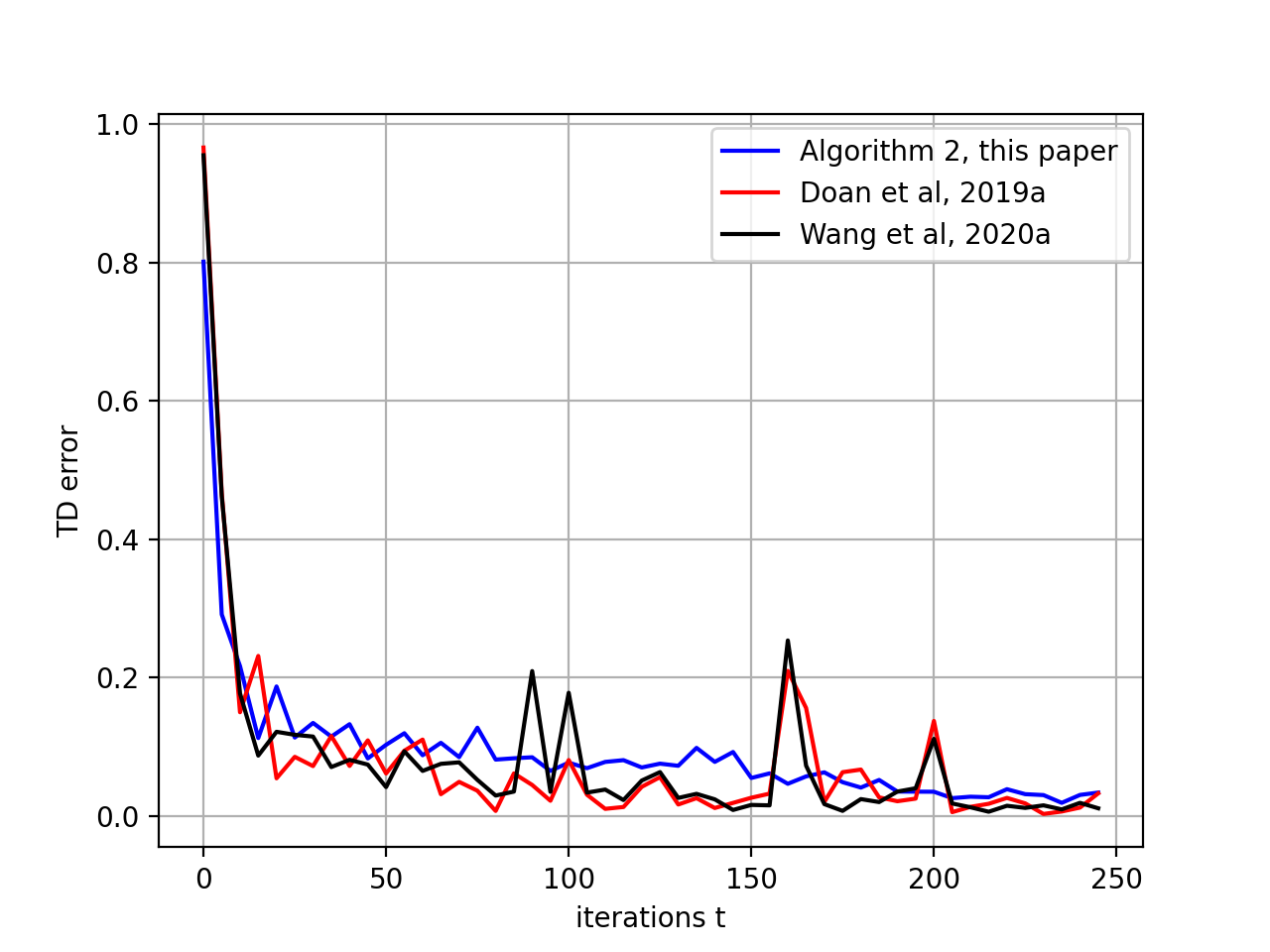}}
\subfigure[MountainCarCont-v0]{
\label{Fig.sub.5'}
\includegraphics[width=0.45\textwidth]{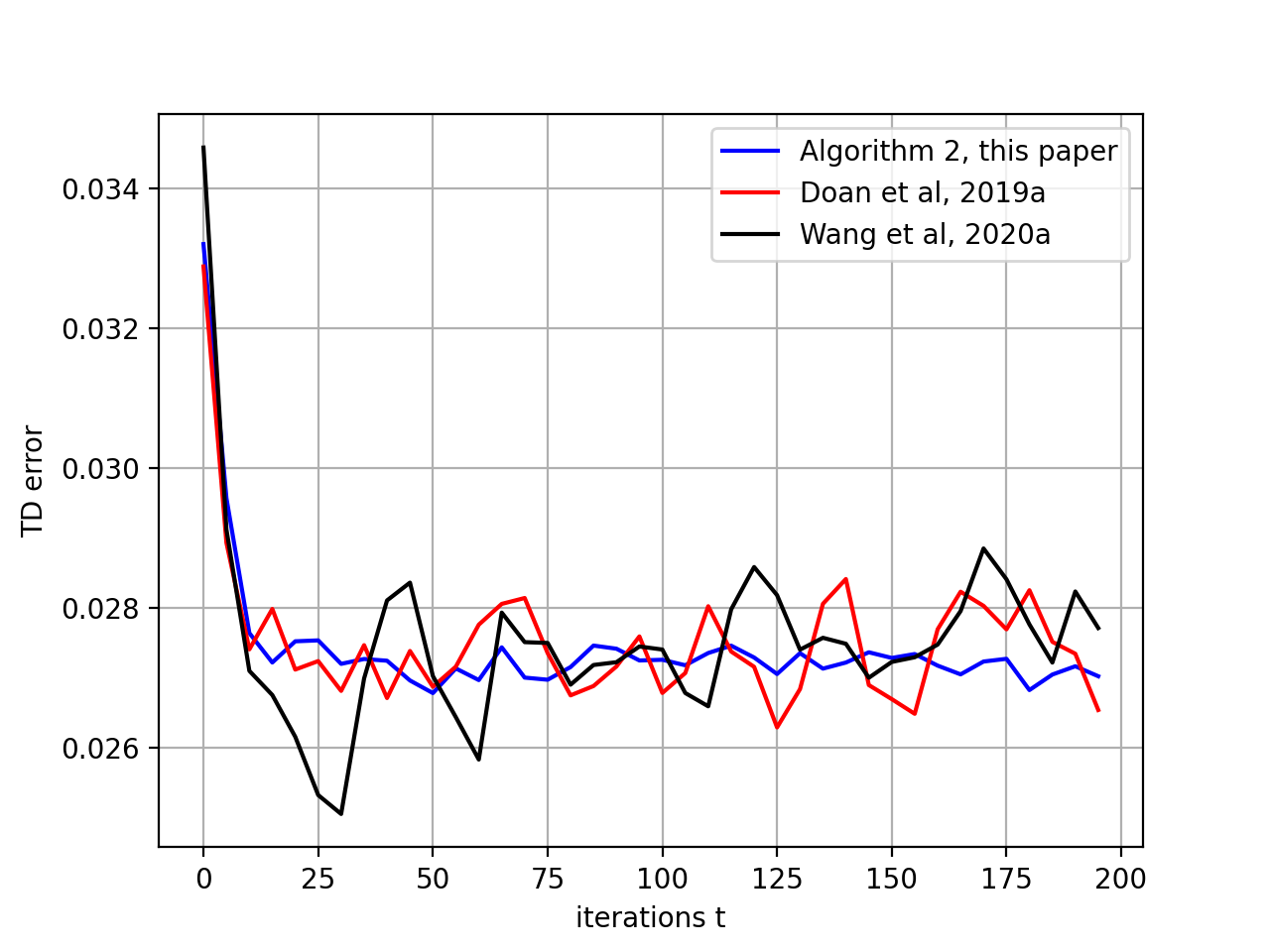}}
\subfigure[Pendulum-v1]{
\label{Fig.sub.6'}
\includegraphics[width=0.45\textwidth]{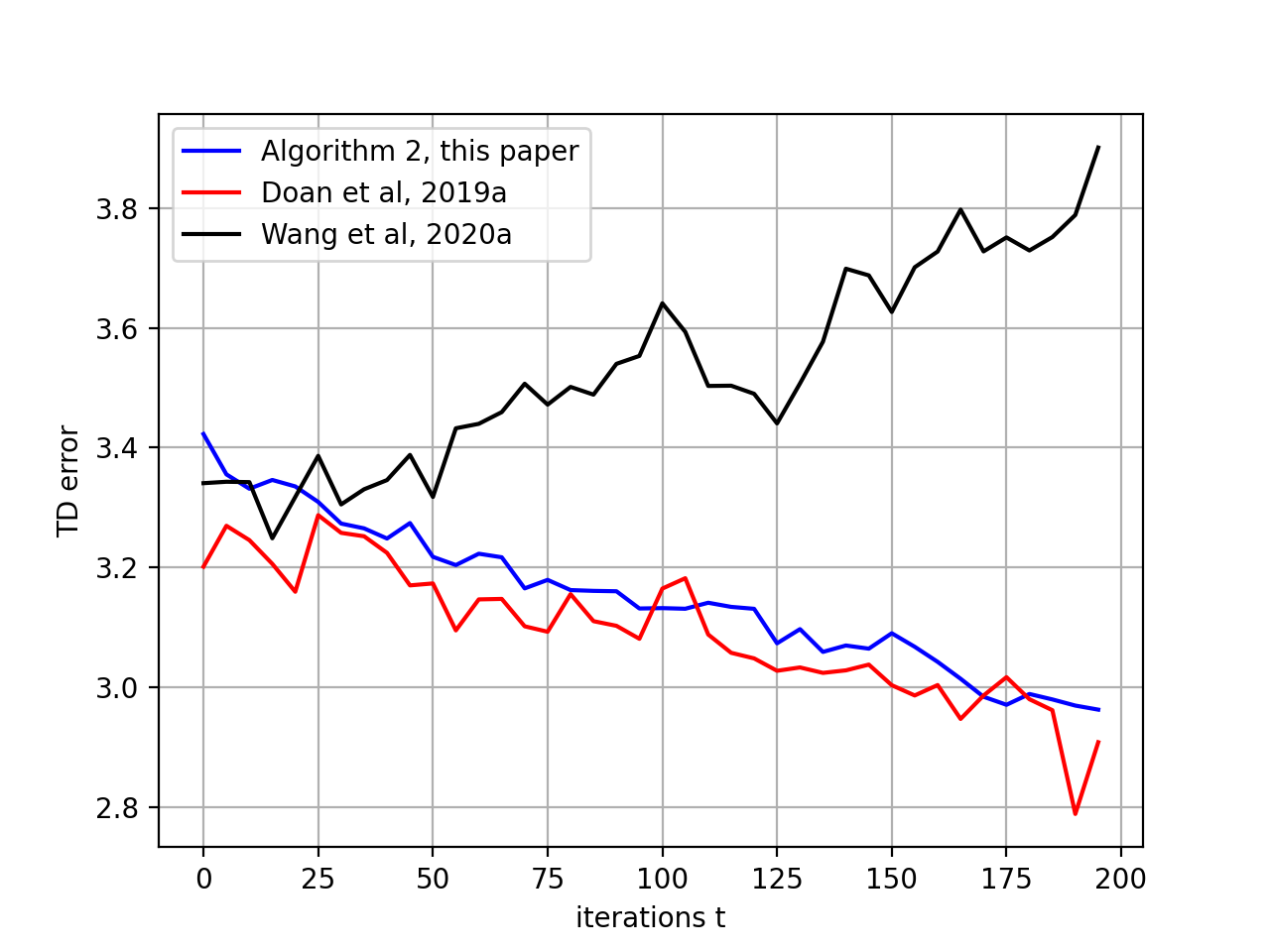}}
\caption{Comparison of our method to the previous literature in the same setting as Figure \ref{fig:var}. Graphs show the TD error vs iteration count. }
\label{fig:compare_six}
\end{figure}

\section{Conclusion}
We have presented convergence results for  distributed TD($0$) with linear function approximation. Our results improve on the previous literature both in terms of utilizing almost no communication: only one run of average consensus is needed. The convergence bounds we derive match state-of-the-art analysis of TD($0$) or reduce the variance by a factor of $N$ when the nodes generate their samples independently.


\clearpage
\bibliography{references}
\bibliographystyle{apalike}

\clearpage

\appendix

\section*{Supplementary Information} 

We now provide proofs of the theorems in the main text of the paper. We begin with a sequence of definitions, notation, and simple observations we will use later.

\section{Notations and Preliminary Results}

\subsection{Linear Equations Satisfied by the Fixed Point} Let $X(t) = \left(s(t),s'(t),r(t) \right)$ be a triple that generates $s(t)$ with steady-state distribution $\pi_s$ and generates $s'(t)$, $r(t)$ from the MDP. It is well-known \cite{tsitsiklis1997analysis} that the limit point $\theta^*$ of centralized TD($0$) is the unique solution of the linear system
\begin{equation} \label{eq:ls} A \theta = b ,
\end{equation} where
\begin{equation}
 A = E [A (X(t))] =E \left[ \phi(s(t)) \left( \gamma \phi(s'(t)) - \phi(s(t)) \right)^T \right] \label{eq:def_A}
\end{equation} and 
\begin{equation*}
 b = E [b (X(t))]=E \left[ r(t) \phi(s(t)) \right].
\end{equation*}


We give a quick generalization of this to the global state in the following lemma. 

\begin{lemma}
Suppose that Assumption \ref{ass:mc}-\ref{ass:features} hold and suppose that the iterates $\{\theta_v(t)\}_{v \in \mathcal{V}} $ are generated by Algorithm \ref{algo:doan} where the step-size sequence $\alpha_t$ is positive,
nonincreasing, and satisfies $\sum_{t=0}^{\infty} \alpha_t = \infty$ and $\sum_{t=0}^{\infty} \alpha_t^2 < \infty$. Then $\bar{\theta}(t)$ converges to $\theta_{\rm gl}^*$ with probability 1, where $\theta_{\rm gl}^*$ is the unique solution of equation 
\begin{equation}
 A \theta = \frac{1}{N}\sum_{v \in \mathcal{V}} b_{v}, \label{eq:dls}
\end{equation}where
\begin{equation}
 b_{v} = E \left[ r_v(t) \phi(s(t)) \right]. \label{eq:def_b}
\end{equation}
\end{lemma}

\begin{proof}
Let $\bar{\theta}(t) = (1/N) \sum_{v} \theta_v(t)$ be the average of the iterates for all the agents in the network. Following the update Eq. (\ref{eq:DTD}), we have
\begin{align*}
 \bar{\theta}(t+1) 
 =& \bar{\theta}(t) +\alpha_t \frac{1}{N} \sum_{v=1}^N \delta_v(t) \phi(s(t))\\
 =& \bar{\theta}(t) +\alpha_t \left( \frac{1}{N}\sum_{v=1}^N r_v(t) \phi(s(t))  - \phi(s(t)) \left(\phi(s(t))- \gamma \phi(s'(t))\right)^T \bar{\theta}(t) \right)
\end{align*}
In other words, $\bar{\theta}(t)$ follows the TD($0$) recursion with $(1/N) \sum_{v} E[r_v(t)]$ as the reward. We thus apply Theorem 2 of \citet{tsitsiklis1997analysis} to obtain this lemma. 

\end{proof}

\subsection{The Expectation of the TD Direction} 
For agent $v$, as shown in Eq. (\ref{eq:DTD}), the direction of the distributed TD($0$) with global state update at iteration $t$ is $\delta_v(t)\phi(s(t))$. We split it into two terms $$\delta_v(t) \phi(s(t)) = h_{v}(t) + m_{v}(t),$$ where 
\begin{align}
 h_{v}(t) & = b_{v} - A\theta_v(t), \label{eq:hv}\\
 m_{v}(t) &= \delta_v(t)\phi(s(t)) - h_{v}(t). \label{eq:mv}
\end{align} It is worth mentioning that, for a given $\theta_v(t)$, the quantity $h_{v}(t)$ can be interpreted as the conditional expectation of the direction $\delta_v(t)\phi(s(t))$:
\begin{equation*}
 h_{v}(t) = E \left[\delta_v(t)\phi(s(t)) | \theta_v(t) \right].
\end{equation*}
Indeed,
\begin{align}
  E \left[\delta_v(t)\phi(s(t)) | \theta_v(t) \right] 
 = & E \left[\left(r_v(t) - \left(\phi(s(t))- \gamma \phi(s'(t))\right)^T \theta_v(t)\right)\phi(s(t)) | \theta_v(t) \right] \notag \\
 = & E \left[r_v(t) \phi(s(t))\right]  - E \left[\phi(s(t))\left(\phi(s(t))- \gamma \phi(s'(t))\right)^T \right] \theta_v(t) \notag \\
 = & b_{v} - A\theta_v(t), \label{eq:conditional}
\end{align}where the second equality follows because $r_v(t)$, $\phi(s(t)), \phi(s'(t))$ have the same distribution regardless of $\theta_v(t)$. 

It is immediate then that $m_{v}(t)$ represents the error, i.e., the difference between the direction $\delta_v(t)\phi(s(t))$ and its conditional expectation. Hence, the update of Eq. (\ref{eq:DTD}) for TD($0$) can be written as:
\begin{equation}
 {\theta}_v({t+1}) = \theta_v(t) + \alpha_t (h_{v}(t) +m_{v}(t)). \label{eq:DTD0}
\end{equation}

Let $\Theta(t)$ denote the matrix whose $v$-th row is the estimates of agent $v$ at time $t$, i.e., $\theta^T_v(t)$. Thus,
\begin{equation*}
 \Theta(t) = \begin{bmatrix}
 \text{---}& \theta_1^T(t) &\text{---}\\ 
 \cdots& \cdots &\cdots \\ 
 \text{---}& \theta_N^T(t) &\text{---}
\end{bmatrix}\in \mathbb{R}^{N \times K}.
\end{equation*}
The matrix form of Algorithm \ref{algo:doan} then can be written as
\begin{equation}
 \Theta({t+1}) =\Theta({t}) + \alpha_t \left[H(t) +M(t)\right],\label{eq:matrix_update}
\end{equation}where $H(t)$, $M(t)$ are the matrices, whose
$v$-th rows are $h_v^T(t)$, $m_v^T(t)$ respectively.

Let us adopt the convention that given a collection of vectors, one for each agent in the network, putting a bar will denote their average. Then 
\begin{equation}\label{eq:avg_update}
 \bar{\theta}({t+1}) = \bar{\theta}({t}) + \alpha_t \left[\bar{h}(t) + \bar{m}(t)\right].
\end{equation}

Our next proposition introduces some useful properties of $\bar{h}(t)$ and $\bar{m}(t)$. 

\begin{proposition} \label{pro:bar}
Suppose Assumptions \ref{ass:mc}-\ref{ass:features} hold, and suppose that $\{\theta_v(t)\}_{v \in \mathcal{V}} $ are generated by Algorithm \ref{algo:doan}. Then,

(a) $\bar{h}(t)$ is a linear function of $\bar{\theta}(t)$: $$\bar{h}(t) = \bar{b} -A \bar{\theta}(t),$$ where $A,b_v$ are defined in Eq. (\ref{eq:def_A}) and Eq. (\ref{eq:def_b}), and $\bar{b}= \frac{1}{N} \sum_{v \in \mathcal{V}} b_{v}$;

(b) The conditional expectation of $\bar{m}(t)$ given $\Theta(t)$ is equal to zero:
\begin{equation}
 E[\bar{m}(t)| \Theta({t}) ] = 0. \label{eq:con_exp}
\end{equation}

\end{proposition}

\begin{proof}[Proof of Proposition \ref{pro:bar}]
(a) Following the definition of $\bar{h}(t)$, it can be observed that:
\begin{equation*}
 \bar{h}(t)  = \frac{1}{N} \sum_{v \in \mathcal{V}} h_v(t)= \frac{1}{N} \sum_{v \in \mathcal{V}} (b_{v} - A\theta_v(t))= \bar{b} -A \bar{\theta}(t),
\end{equation*} where $\bar{b}= \frac{1}{N} \sum_{v \in \mathcal{V}} b_{v}$.

(b) Recall that $h_v(t)$ is exactly the conditional expectation of $\delta_v(t)\phi(s(t))$ given $\theta_v(t)$. Actually, the more general statement 
\[ E [ \delta_v(t) \phi(s(t)) | \Theta(t)] = h_v(t),
\] is true; for a proof, one can simply repeat all the steps of Eq. (\ref{eq:conditional}) replacing $\theta_v(t)$ by $\Theta(t)$. 

Furthermore, $m_v(t)$ is defined as $$m_v(t) = \delta_v(t)\phi(s(t)) - h_v(t),$$ by Eq. (\ref{eq:mv}). Therefore: $\theta_v({t})$:
\begin{align*}
 E \left [m_v(t) | \Theta(t) \right] &= E \left [\delta_v(t)\phi(s(t)) - h_v(t) | \Theta(t) \right]\\
 & = E \left [\delta_v(t)\phi(s(t))| \Theta(t) \right] - h_v(t) \\
 & = 0
\end{align*} By the definition of $\bar{m}(t)$, we then have
\begin{equation*}
 E[\bar{m}(t)| \Theta({t}) ] = 0. 
\end{equation*}
\end{proof}

\subsection{Averages of TD Updates}

In this subsection, we introduce an equality satisfied by the inner product between the averages of TD updates throughout the agents in the network and the direction from the averages of the iterates throughout the network to the fixed point of distributed TD algorithm.

\begin{lemma}\label{lem:inner_product}
Suppose Assumptions \ref{ass:mc}-\ref{ass:features} hold. Further suppose that $\{ \theta_v(t)\}_{v \in \mathcal{V}}$ are generated by Algorithm \ref{algo:doan}. For any integer $t \geq 0$, we have that 
\begin{equation}\label{eq:inner_product}
    E\left[ \left[\bar{h}(t)+ \bar{m}(t)\right]^T (\theta_{\rm gl}^* - \bar{\theta}({t}))\right]
 = E\left[ (1-\gamma)\left\|V_{\theta_{\rm gl}^*} - V_{\bar{\theta}(t)} \right\|_{D}^2 + \gamma \left\|V_{\theta_{\rm gl}^*} - V_{\bar{\theta}(t)} \right\|_{{\rm Dir}}^2 \right]
\end{equation}
\end{lemma}

\begin{proof}

\begin{align*}
  E\left[ \left[\bar{h}(t)+ \bar{m}(t)\right]^T (\theta_{\rm gl}^* - \bar{\theta}({t}))\right] 
 = & E\left[ E\left[\left(\bar{h}(t) + \bar{m}(t)\right)^T (\theta_{\rm gl}^* - \bar{\theta}({t}))| \Theta({t})\right] \right] \notag\\
 = & E\left[\bar{h}^T(t) (\theta_{\rm gl}^* - \bar{\theta}({t})) \right], 
\end{align*}where in the second equation, we use Eq. (\ref{eq:con_exp}). Here, by Proposition \ref{pro:bar} part (a), we have that $\bar{h}(t)$ is a linear function of $\bar{\theta}(t)$, i.e., $\bar{h}(t) = \bar{b} -A \bar{\theta}(t)$. Furthermore, if we let $ \bar{h}(\theta)$ denote the linear function $\bar{b} -A {\theta}$, we can obtain that $\bar{h}(\theta_{\rm gl}^*)=0$.

Corollary 1 in \citet{liu2020temporal} states that for any $\theta \in \mathbb{R}^K$,
\begin{equation*}
    (\theta^* - \theta)^T \bar{g}(\theta) = (1-\gamma)\|V_{\theta^*}-V_{\theta}\|_{D}^2 + \gamma \|V_{\theta^*}-V_{\theta}\|_{{\rm Dir}}^2,
\end{equation*}where $\bar{g}(\theta)$ in that paper denote the steady-state expectation of $r(s,s')\phi(s(t)) - \phi(s(t)) \left(\phi(s)- \gamma \phi(s')\right)^T \theta $, which is indeed $ b - A \theta$ and $\theta^*$ is the limit point of centralized TD(0) method such that $\bar{g}(\theta^*)=0$.

 Now applying Corollary 1 in \citet{liu2020temporal}, we can obtain Eq.(\ref{eq:inner_product}).

\end{proof}

\section{Proof of Theorem \ref{thm:1}}

With the above preliminaries in place, we can now begin the proof of Theorem \ref{thm:1}. Our first step is to analyze the recurrence relation satisfied by the averages of the iterates throughout the network. 


\begin{lemma}\label{lem:1}
Suppose Assumptions \ref{ass:mc}-\ref{ass:features} hold. Further suppose that $\{ \theta_v(t)\}_{v \in \mathcal{V}}$ are generated by Algorithm \ref{algo:doan}. For any integer $t \geq 0$, we have that 

\begin{align*}
 E\left[\left\|\bar{\theta}({t+1}) - \theta_{\rm gl}^*\right\|_2^2 \right]
 \leq & E\left[ \left\|\bar{\theta}({t}) - \theta_{\rm gl}^*\right\|_2^2 \right] +\alpha_t^2 \left[{2\bar{\sigma}^2} + 8 \left\| V_{\theta_{\rm gl}^*} - V_{\bar{\theta}(t)} \right\|_{D}^2 \right] \\& - 2 \alpha_t E\left[ (1-\gamma) \left\|V_{\theta_{\rm gl}^*} - V_{\bar{\theta}(t)} \right\|_{D}^2+ \gamma \left\|V_{\theta_{\rm gl}^*} - V_{\bar{\theta}(t)} \right\|_{{\rm Dir}}^2 \right],
\end{align*}

\end{lemma}

\begin{proof}[Proof of Lemma \ref{lem:1}]
By Eq. (\ref{eq:avg_update}), we have
\begin{align*}
  \left\|\bar{\theta}({t+1}) - \theta_{\rm gl}^*\right\|_2^2
 = & \left\|\bar{\theta}({t}) + \alpha_t \left[\bar{h}(t) + \bar{m}(t)\right] - \theta_{\rm gl}^*\right\|_2^2\\
 = & \left\|\bar{\theta}({t}) - \theta_{\rm gl}^*\right\|_2^2 + 2 \alpha_t \left[\bar{h}(t) + \bar{m}(t)\right]^T (\bar{\theta}({t}) - \theta_{\rm gl}^*) + \alpha_t^2 \left\|\bar{h}(t) + \bar{m}(t)\right\|_2^2 .
\end{align*}

Taking expectations we obtain that
\begin{equation}
  E\left[\left\|\bar{\theta}({t+1}) - \theta_{\rm gl}^*\right\|_2^2 \right] 
 = E\left[ \left\|\bar{\theta}({t}) - \theta_{\rm gl}^*\right\|_2^2 \right] + \alpha_t^2E\left[ \left\|\bar{h}(t) + \bar{m}(t)\right\|_2^2 \right]  - 2 \alpha_t E\left[ \left(\bar{h}(t)+ \bar{m}(t)\right)^T \left(\theta_{\rm gl}^* - \bar{\theta}({t})\right)\right] . \label{eq:exp_rec}
\end{equation}

Consider the second term on the right hand side of Eq. (\ref{eq:exp_rec}). Following the definition of $\bar{h}(t)$ and $\bar{m}(t)$, we have that
\begin{equation*}
 E\left[ \left\|\bar{h}(t) + \bar{m}(t)\right\|_2^2 \right] = E\left[ \left\|\frac{1}{N} \sum_{v \in \mathcal{V}} \delta_v(t)\phi(s(t)) \right\|_2^2\right].
\end{equation*}
Plugging in the expression for TD error $\delta_v(t)$ with Eq. (\ref{eq:delta}), we obtain

\begin{align*}
 & E\left[ \left\|\bar{h}(t) + \bar{m}(t)\right\|_2^2 \right] \\
 =& E\left[ \left\| \frac{1}{N} \sum_{v \in \mathcal{V}} r_v(t) \phi(s(t)) \right. \right. \left. \left. - \frac{1}{N} \sum_{v \in \mathcal{V}} \phi(s(t))\left(\phi(s(t))- \gamma \phi(s'(t))\right)^T\theta_v(t) \right\|_2^2\right] \\
 =& E\left[ \left\| \frac{1}{N} \sum_{v \in \mathcal{V}} \left[ {r_v}(t) - \left(\phi(s(t))- \gamma \phi(s'(t))\right)^T \theta_{\rm gl}^*\right]\phi(s(t))  -\frac{1}{N} \sum_{v \in \mathcal{V}} \phi(s(t))\left(\phi(s(t))- \gamma \phi(s'(t))\right)^T( \theta_v(t) -\theta_{\rm gl}^*) \right\|_2^2\right] 
\end{align*}

Denote
$$ \bm{a}^* =\frac{1}{N} \sum_{v \in \mathcal{V}} \left[ {r_v}(t) - \left(\phi(s(t))- \gamma \phi(s'(t))\right)^T \theta_{\rm gl}^*\right]\phi(s(t)), $$ $$\bm{b}^* = \frac{1}{N} \sum_{v \in \mathcal{V}} \phi(s(t))\left(\phi(s(t))- \gamma \phi(s'(t))\right)^T( \theta_v(t) -\theta_{\rm gl}^*).$$ Using inequality $\|\bm{a}^*-\bm{b}^*\|^2 \leq 2\|\bm{a}^*\|^2 + 2\|\bm{b}^*\|^2$, we obtain
\begin{equation}
 E\left[ \left\|\bar{h}(t) + \bar{m}(t)\right\|_2^2 \right] \leq 2E\left[\|\bm{a}^*\|^2\right] + 2E\left[\|\bm{b}^*\|^2\right]. \label{eq:2norm}
\end{equation}
We first bound $E\left[\|\bm{a}^*\|^2\right]$:

\begin{align}
 E\left[\|\bm{a}^*\|^2\right] = & E\left[\left\|  \frac{1}{N} \sum_{v \in \mathcal{V}} {r_v}(t)\phi(s(t))  - \left( \left(\phi(s(t))- \gamma \phi(s'(t))\right)^T \theta_{\rm gl}^* \right) \phi(s(t)) \right\|^2\right] \notag\\
 \leq & E\left[ \left(\frac{1}{N} \sum_{v \in \mathcal{V}} {r_v}(t) - \left(\phi(s(t))- \gamma \phi(s'(t))\right)^T \theta_{\rm gl}^*\right) ^2\right]
 = \bar{\sigma}^2, \label{eq:bounda}
\end{align}where the inequality follows by Assumption \ref{ass:features} and the equality is just the definition of $\bar{\sigma}^2$. 

The next step is to find the bound for $E\left[\|\bm{b}^*\|^2\right]$:

\begin{align}
 E\left[\|\bm{b}^*\|^2\right] 
 = & E\left[\left\| \frac{1}{N} \sum_{v \in \mathcal{V}} \phi(s(t))\left(\phi(s(t))- \gamma \phi(s'(t))\right)^T( \theta_v(t) -\theta_{\rm gl}^*) \right\|^2\right]\notag \\
 = & E\left[\left\| \phi(s(t))\left(\phi(s(t))- \gamma \phi(s'(t))\right)^T( \bar{\theta}(t) -\theta_{\rm gl}^*) \right\|^2\right]\notag \\
 \leq & 4 \left\| V_{\theta_{\rm gl}^*} - V_{\bar{\theta}(t)} \right\|_{D}^2, \label{eq:boundb}
\end{align}where the last line follows from the proof of Lemma 5 in \citet{bhandari2018finite}.

Plugging Eq. (\ref{eq:bounda}) and Eq. (\ref{eq:boundb}) into Eq. (\ref{eq:2norm}), we get
\begin{equation}
 E\left[ \left\|\bar{h}(t) + \bar{m}(t)\right\|_2^2 \right] \leq {2\bar{\sigma}^2} + 8 \left\| V_{\theta_{\rm gl}^*} - V_{\bar{\theta}(t)} \right\|_{D}^2. \label{eq:2nd}
\end{equation}

The argument we just made bounds one of the terms in Eq. (\ref{eq:exp_rec}). We next consider a different term in the same equation, namely we consider the third term on the right hand side of Eq. (\ref{eq:exp_rec}) which satisfies the equality in Lemma \ref{lem:inner_product}:
\begin{equation}
  E\left[ \left[\bar{h}(t)+ \bar{m}(t)\right]^T (\theta_{\rm gl}^* - \bar{\theta}({t}))\right]
 = E\left[ (1-\gamma)\left\|V_{\theta_{\rm gl}^*} - V_{\bar{\theta}(t)} \right\|_{D}^2 + \gamma \left\|V_{\theta_{\rm gl}^*} - V_{\bar{\theta}(t)} \right\|_{{\rm Dir}}^2 \right]\label{eq:inner_product_2}
\end{equation}

Finally, combining equations Eq. (\ref{eq:exp_rec}), Eq. (\ref{eq:inner_product_2}), and Eq. (\ref{eq:2nd}), we obtain
\begin{align*}
 E\left[\left\|\bar{\theta}({t+1}) - \theta_{\rm gl}^*\right\|_2^2 \right]
 \leq & E\left[ \left\|\bar{\theta}({t}) - \theta_{\rm gl}^*\right\|_2^2 \right] +\alpha_t^2 \left[{2\bar{\sigma}^2} + 8 \left\| V_{\theta_{\rm gl}^*} - V_{\bar{\theta}(t)} \right\|_{D}^2 \right]\\
 & - 2 \alpha_t E\left[ (1-\gamma) \left\|V_{\theta_{\rm gl}^*} - V_{\bar{\theta}(t)} \right\|_{D}^2+ \gamma \left\|V_{\theta_{\rm gl}^*} - V_{\bar{\theta}(t)} \right\|_{{\rm Dir}}^2 \right],
\end{align*}
which is what we needed to show.
\end{proof}

With this lemma in place, we next prove Theorem \ref{thm:1}.

\begin{proof}[Proof of Theorem \ref{thm:1}]
Starting with the statement of Lemma \ref{lem:1}, which we reproduce here for convenience,
\begin{align}
  E\left[\left\|\bar{\theta}({t+1}) - \theta_{\rm gl}^*\right\|_2^2 \right]
 \leq &E\left[ \left\|\bar{\theta}({t}) - \theta_{\rm gl}^*\right\|_2^2 \right] +\alpha_t^2 \left[{2\bar{\sigma}^2} + 8 \left\| V_{\theta_{\rm gl}^*} - V_{\bar{\theta}(t)} \right\|_{D}^2 \right] \notag \\
 & - 2 \alpha_t E\left[ (1-\gamma) \left\|V_{\theta_{\rm gl}^*} - V_{\bar{\theta}(t)} \right\|_{D}^2+ \gamma \left\|V_{\theta_{\rm gl}^*} - V_{\bar{\theta}(t)} \right\|_{{\rm Dir}}^2 \right].\label{eq:rec}
\end{align} we will plug in different choices of step-size. 

\textbf{Proof of part (a):} We consider a constant step-size sequence $\alpha_0 = \cdots = \alpha_T \leq (1- \gamma)/8$. Let $\alpha$ denote this constant step-size. Since $\alpha \leq (1- \gamma)/8$, it follows that 
\begin{equation*}
 8 \alpha^2 - 2 \alpha (1-\gamma) \leq - \alpha (1-\gamma).
\end{equation*}

Plugging this into Eq. (\ref{eq:rec}) and rearranging, 
\begin{align*}
 & E\left[\left\|\bar{\theta}({t+1}) - \theta_{\rm gl}^*\right\|_2^2 \right] \notag\\
 \leq &E\left[ \left\|\bar{\theta}({t}) - \theta_{\rm gl}^*\right\|_2^2 \right] - \alpha (1-\gamma) E \left[ \left\|V_{\theta_{\rm gl}^*} - V_{\bar{\theta}(t)} \right\|_{D}^2\right]+ {2\alpha^2 \bar{\sigma}^2} - 2\alpha \gamma E\left[ \left\|V_{\theta_{\rm gl}^*} - V_{\bar{\theta}(t)} \right\|_{{\rm Dir}}^2 \right] \\
 \leq & E\left[ \left\|\bar{\theta}({t}) - \theta_{\rm gl}^*\right\|_2^2 \right] - \alpha (1-\gamma) E \left[ \left\|V_{\theta_{\rm gl}^*} - V_{\bar{\theta}(t)} \right\|_{D}^2\right] + {2\alpha^2 \bar{\sigma}^2}\\
 \leq & \left(1-\alpha (1-\gamma) \omega \right) E\left[ \left\|\bar{\theta}({t}) - \theta_{\rm gl}^*\right\|_2^2 \right]+ {2\alpha^2 \bar{\sigma}^2},\\
\end{align*}where the second inequality follows that $\left\|V_{\theta_{\rm gl}^*} - V_{\bar{\theta}(t)} \right\|_{{\rm Dir}}^2$ is non-negative and the third inequality uses Lemma 1 in \citet{bhandari2018finite} which states that $$\sqrt{\omega} \|\theta\|_2 \leq \|V_{\theta}\|_{D} \leq \|\theta\|_2 .$$
Iterating this inequality establishes that after $T$ iterations
\begin{align*}
 E\left[\left\|\bar{\theta}(T) - \theta_{\rm gl}^*\right\|_2^2 \right] 
 \leq & \left(1-\alpha (1-\gamma) \omega \right)^T E\left[ \left\|\bar{\theta}({0}) - \theta_{\rm gl}^*\right\|_2^2 \right]+ 2\alpha^2 \bar{\sigma}^2 \sum_{t=0}^{T-1} \left(1-\alpha (1-\gamma) \omega \right)^t\\
 \leq & e^{-\alpha (1-\gamma) \omega T} E\left[ \left\|\bar{\theta}({0}) - \theta_{\rm gl}^*\right\|_2^2 \right] + \frac{2\alpha \bar{\sigma}^2}{ (1-\gamma) \omega},
\end{align*}where the last inequality follows because $1-\alpha (1-\gamma) \omega \leq e^{-\alpha (1-\gamma) \omega }$ and the standard formula for the sum of a geometric series.

\textbf{Proof of part (b):} We now take the step-size $\alpha_0 = \cdots = \alpha_T =\frac{1}{\sqrt{T}}$. Since the step-size is constant once we fix $T$, we can denote it by $\alpha$. Since $T \geq \frac{64}{ (1-\gamma)^2}$, it can be observed that $\alpha = \frac{1}{\sqrt{T}} \leq \frac{1-\gamma}{8}$. 
Plugging this into Eq. (\ref{eq:rec}) and rearranging it, we obtain
\begin{equation*}
  \alpha E\left[ (1-\gamma) \left\|V_{\theta_{\rm gl}^*} - V_{\bar{\theta}(t)} \right\|_{D}^2+ 2 \gamma \left\|V_{\theta_{\rm gl}^*} - V_{\bar{\theta}(t)} \right\|_{{\rm Dir}}^2 \right]
 \leq E\left[ \left\|\bar{\theta}({t}) - \theta_{\rm gl}^*\right\|_2^2 \right] - E\left[\left\|\bar{\theta}({t+1}) - \theta_{\rm gl}^*\right\|_2^2 \right]+2 \alpha^2 \bar{\sigma}^2.
\end{equation*}
Summing over $t$ gives
\begin{align*}
 \sum_{t=0}^{T-1} \alpha E\left[ (1-\gamma) \left\|V_{\theta_{\rm gl}^*} - V_{\bar{\theta}(t)} \right\|_{D}^2+ 2 \gamma \left\|V_{\theta_{\rm gl}^*} - V_{\bar{\theta}(t)} \right\|_{{\rm Dir}}^2 \right]
 \leq &E\left[ \left\|\bar{\theta}({0}) - \theta_{\rm gl}^*\right\|_2^2 \right] - E\left[\left\|\bar{\theta}({T}) - \theta_{\rm gl}^*\right\|_2^2 \right]+2T \alpha^2 \bar{\sigma}^2\\
 \leq & E\left[ \left\|\bar{\theta}({0}) - \theta_{\rm gl}^*\right\|_2^2 \right] +2T \alpha^2 \bar{\sigma}^2,
\end{align*}
Dividing by $ \alpha$ on both sides, we obtain:
\begin{equation*}
  \sum_{t=0}^{T-1} E\left[ (1-\gamma) \left\|V_{\theta_{\rm gl}^*} - V_{\bar{\theta}(t)} \right\|_{D}^2+ 2 \gamma \left\|V_{\theta_{\rm gl}^*} - V_{\bar{\theta}(t)} \right\|_{{\rm Dir}}^2 \right]
 \leq  \frac{1}{\alpha}E\left[ \left\|\bar{\theta}({0}) - \theta_{\rm gl}^*\right\|_2^2 \right] +2 T \alpha \bar{\sigma}^2.
\end{equation*}
Finally, recall our notation $\hat{\theta}(T)=\frac{1}{T} \sum_{t=1}^T \bar{\theta}(t) $. Then, by convexity
\begin{align}
 E\left[(1-\gamma) \left\|V_{\theta_{\rm gl}^*} - V_{\hat{\theta}(T)} \right\|_{D}^2 + 2 \gamma \left\|V_{\theta_{\rm gl}^*} - V_{\hat{\theta}(T)} \right\|_{{\rm Dir}}^2 \right] 
 \leq& \frac{1}{T} \sum_{t=1}^T E\left[ (1-\gamma) \left\|V_{\theta_{\rm gl}^*} - V_{\bar{\theta}(t)} \right\|_{D}^2 + 2 \gamma \left\|V_{\theta_{\rm gl}^*} - V_{\bar{\theta}(t)} \right\|_{{\rm Dir}}^2 \right] \notag \\
 \leq & \frac{1}{\alpha T }E\left[ \left\|\bar{\theta}({0}) - \theta_{\rm gl}^*\right\|_2^2 \right] +2 \alpha \bar{\sigma}^2 . \label{eq:alpha}
\end{align}
Plugging in that $\alpha =\frac{1}{\sqrt{T}}$, we have that
\begin{align*}
 E\left[(1-\gamma) \left\|V_{\theta_{\rm gl}^*} - V_{\hat{\theta}(T)} \right\|_{D}^2 + 2 \gamma \left\|V_{\theta_{\rm gl}^*} - V_{\hat{\theta}(T)} \right\|_{{\rm Dir}}^2 \right] 
 \leq & \frac{1}{ \sqrt{T} }E\left[ \left\|\bar{\theta}({0}) - \theta_{\rm gl}^*\right\|_2^2 \right] +\frac{2 \bar{\sigma}^2}{\sqrt{T}} \\
 = & \frac{E\left[ \left\|\bar{\theta}({0}) - \theta_{\rm gl}^*\right\|_2^2 \right] + 2 \bar{\sigma}^2 }{\sqrt{T}}.
\end{align*}

\textbf{Proof of part (c):} From $\alpha_t \leq \alpha_0 = \frac{\alpha}{\tau}= \frac{1-\gamma}{8}$, we have that for $t \geq 0$, it can once again be observed that $$8 \alpha_t^2 - 2 \alpha_t (1-\gamma) \leq - \alpha_t (1-\gamma).$$ Applying this to Eq. (\ref{eq:rec}), we have

\begin{align*}
 E \left[ \left\|\bar{\theta}({t+1}) - \theta_{\rm gl}^*\right\|_2^2 \right] 
 \leq & E\left[ \left\|\bar{\theta}({t}) - \theta_{\rm gl}^*\right\|_2^2 \right]+ 2 \alpha_t^2 \bar{\sigma}^2 - \alpha_t (1-\gamma) E \left[ \left\|V_{\theta_{\rm gl}^*} - V_{\bar{\theta}(t)} \right\|_{D}^2\right] - 2 \alpha_t \gamma E \left[ \left\|V_{\theta_{\rm gl}^*} - V_{\bar{\theta}(t)} \right\|_{\rm Dir}^2\right]\\
 \leq &E\left[ \left\|\bar{\theta}({t}) - \theta_{\rm gl}^*\right\|_2^2 \right]+ 2 \alpha_t^2 \bar{\sigma}^2 - \alpha_t (1-\gamma) E \left[ \left\|V_{\theta_{\rm gl}^*} - V_{\bar{\theta}(t)} \right\|_{D}^2\right],
\end{align*}
Applying Lemma 1 in \citet{bhandari2018finite} stating that $\sqrt{\omega} \|\theta\|_2 \leq \|V_{\theta}\|_{{D}}$, we obtain
\begin{equation*}
  E \left[ \left\|\bar{\theta}({t+1}) - \theta_{\rm gl}^*\right\|_2^2 \right] 
 \leq ( 1- \alpha_t (1-\gamma) \omega) E\left[ \left\|\bar{\theta}({t}) - \theta_{\rm gl}^*\right\|_2^2 \right] + 2 \alpha_t^2 \bar{\sigma}^2.
\end{equation*}
Wwe will next prove by induction that this last inequality implies that $$ E\left[\left\|\bar{\theta}({t}) - \theta_{\rm gl}^*\right\|_2^2 \right] \leq \frac{\zeta}{t+\tau},$$ where $\zeta = \max\left\{ 2 \alpha^2 \bar{\sigma}^2, \tau \left\|\bar{\theta}({0}) - \theta_{\rm gl}^*\right\|_2^2 \right\}$. 

Indeed, the assertion clearly holds at $t=0$. Suppose that the assertion holds at time $t$, i.e., suppose that $ E\left[\left\|\bar{\theta}({t}) - \theta_{\rm gl}^*\right\|_2^2 \right] \leq \frac{\zeta}{t+\tau}$. Then,
\begin{align*}
 E \left[ \left\|\bar{\theta}({t+1}) - \theta_{\rm gl}^*\right\|_2^2 \right] 
 \leq & ( 1- \alpha_t (1-\gamma) \omega) \cdot \frac{\zeta}{t+ \tau} + 2 \alpha_t^2 \bar{\sigma}^2 \\
 = & \frac{(t+ \tau)\zeta - \alpha (1-\gamma) \omega \zeta + 2 \alpha^2 \bar{\sigma}^2 }{(t+ \tau)^2} \\
 = & \frac{(t+ \tau-1 )\zeta + \zeta - \alpha (1-\gamma) \omega \zeta + 2 \alpha^2 \bar{\sigma}^2}{(t+ \tau)^2} \\
 = & \frac{(t+ \tau-1 )\zeta - \zeta + 2 \alpha^2 \bar{\sigma}^2}{(t+ \tau)^2} \\
 \leq & \frac{(t+ \tau-1 )\zeta }{(t+ \tau)^2} \leq \frac{(t+ \tau-1 )\zeta }{(t+ \tau)^2 -1 } \\
 = & \frac{\zeta }{t+1+\tau},
\end{align*}where note that, at the first equal sign, we plug in the expression for $\alpha_t$; and the second inequality, we use that $\zeta \geq 2 \alpha^2 \bar{\sigma}^2$.
\end{proof}
Next, we state Eq. (\ref{eq:alpha}) in the proof of part(b) as the following Corollary.
\begin{corollary}
Suppose Assumptions \ref{ass:mc}-\ref{ass:features} hold. Suppose further that $\{ \theta_v(t)\}_{v \in \mathcal{V}}$ are generated by Algorithm \ref{algo:doan} in the global state model. Then, For any constant step-size sequence $\alpha_0 = \cdots = \alpha_T= \alpha \leq (1- \gamma)/8$,
\begin{equation*}
  E\left[(1-\gamma) \left\|V_{\theta_{\rm gl}^*} - V_{\hat{\theta}(T)} \right\|_{D}^2 + 2 \gamma \left\|V_{\theta_{\rm gl}^*} - V_{\hat{\theta}(T)} \right\|_{{\rm Dir}}^2 \right] 
 \leq  \frac{1}{\alpha T }E\left[ \left\|\bar{\theta}({0}) - \theta_{\rm gl}^*\right\|_2^2 \right] +2 \alpha \bar{\sigma}^2 .
\end{equation*}
\end{corollary}

We believe this corollary is of independent interest, as it gives a very clean expression for steady-state error of $O(\alpha \bar{\sigma}^2)$ when evaluated in terms of the metric on the left-hand side. Note that there is no scaling here withe either the number of agents or the spectral gap, either in the final steady-state error or in the convergence time that comes from the first term on the right-hand side. 

We state this as a corollary, rather than a theorem, since our theorem uses the metrics in the previous work, i.e., $||\theta_{\rm gl}^* - \hat{\theta}(T)||_2^2$ in the case of fixed step-size, to make the comparison between this paper and earlier papers clear. However, the relatively clean expression for the final steady-state error in this corollary suggests that, rather than using the distance to the optimal solution as the metric of performance, it is better to use the distances between the corresponding value vectors.

\section{Proof of Theorem \ref{thm:main}}

We now turn to the proof of Theorem \ref{thm:main}. We now assume, for the remainder of this section, that we are analyzing Algorithm \ref{algo:our} in the local state model, subject to Assumptions \ref{ass:mc} and \ref{ass:features}. It turns out that, in many respects, the local state model can be treated analogously to the global state model. 

Our starting point is similar. For a particular agent $v$, as shown in Eq. (\ref{eq:DTD2}), the direction of the distributed TD($0$) update at iteration $t$ is $\delta_v(t)\phi(s_v(t))$. As before, we split this into two terms $$\delta_v(t) \phi(s_v(t)) = h_{v}(t) + m_{v}(t),$$ where 
\begin{align*}
 h_{v}(t) & = b - A\theta_v(t),\\
 m_{v}(t) &= \delta_v(t)\phi(s_v(t)) - h_{v}(t),
\end{align*} where 
\[ A 
 =E_{\mu,s} \left[ \phi(s_v(t)) \left( \gamma \phi(s_v'(t)) - \phi(s_v(t)) \right)^T \right], \] where, note that the right-hand side does not actually depend on $v$ since $s_v(t), s_v'(t)$ have the same joint distribution regardless of $v$, and 
 \[ 
 b = E_{\mu,s} \left[ r_v(t) \phi(s_v(t)) \right] 
 \] where, again, the right-hand side actually does not depend on $v$.

It is worth mentioning that, although here we use the same notation $h_v$ and $m_v$ as in our earlier treatment of the global state model, the definitions are now slightly different since each agent maintains its own state $s_v(t)$.


As before we let $\Theta(t)$ be \begin{equation*} \Theta(t) = \begin{bmatrix}
 \text{---}& \theta_1^T(t) &\text{---}\\ 
 \cdots& \cdots &\cdots \\ 
 \text{---}& \theta_N^T(t) &\text{---}
\end{bmatrix}\in \mathbb{R}^{N \times K}.
\end{equation*} Our first observation is that the quantity $h_{v}(t)$ can be interpreted as the conditional expectation of the update direction:
\begin{equation*}
 h_{v}(t) = E \left[\delta_v(t)\phi(s_v(t)) | \Theta(t) \right].
\end{equation*}
An identical argument as in Eq. (\ref{eq:conditional}) can be carried out since $r_v(t)$, $\phi(s_v(t)), \phi(s_v'(t))$ have the same distribution for all agents $v$. 

Similarly to before, we have the following proposition for the local state model. Recall here our notation of putting a bar to denote the network-wide average. 

\begin{proposition} \label{pro:bar_lc}
Suppose Assumptions \ref{ass:mc}-\ref{ass:features} hold, and suppose that $\{\theta_v(t)\}_{v \in \mathcal{V}} $ are generated by Algorithm \ref{algo:our}. Then,

(a) $\bar{h}(t)$ is a linear function of $\bar{\theta}(t)$: $$\bar{h}(t) = {b} -A \bar{\theta}(t).$$ 

(b) The conditional expectation of $\bar{m}(t)$ given $\Theta(t)$ is equal to zero:
\begin{equation}
 E[\bar{m}(t)| \Theta({t}) ] = 0. \label{eq:con_exp_lc}
\end{equation}

\end{proposition}

The proof of this proposition is essentially identical to the proof of Proposition \ref{pro:bar} and we omit it. 

Our next step is to prove a recurrence relation satisfied by the average of the iterates, stated as the following lemma. The key differences between this lemma and the previously-proved version in the global state model is that the quantity ${\sigma}^2$ that appears in this recursion will now be divided by $N$, at the cost of the addition of an extra term we will have to deal with. Recall that $\theta_{\rm lc}^*$ is the fixed point of TD($0$) on the MDP $(\mathcal{S},\mathcal{V},\mathcal{A}, \mathcal{P}, r, \gamma)$.

\begin{lemma}\label{lem:2}
Suppose Assumptions \ref{ass:mc}-\ref{ass:features} hold. Further suppose that $\{\theta_v\}_{v \in \mathcal{V}}$ are generated by Algorithm \ref{algo:our}. For $t \in \mathbb{N}_0$, we have that 
\begin{align*}
  E\left[\left\|\bar{\theta}({t+1}) - \theta_{\rm lc}^*\right\|_2^2 \right] \leq & E\left[ \left\|\bar{\theta}({t}) - \theta_{\rm lc}^*\right\|_2^2 \right]  +\alpha_t^2 \left(\frac{2{\sigma}^2}{N} + \frac{8}{N} \sum_{v \in \mathcal{V}} E\left[ \|V_{\theta_v(t)} - V_{\theta_{\rm lc}^*} \|_{D}^2\right]\right )\\
 &  - 2 \alpha_t E\left[ (1-\gamma) \left\|V_{\theta_{\rm lc}^*} - V_{\bar{\theta}(t)} \right\|_{D}^2+ \gamma \left\|V_{\theta_{\rm lc}^*} - V_{\bar{\theta}(t)} \right\|_{{\rm Dir}}^2 \right],
\end{align*}
\end{lemma}


\begin{proof}[Proof of Lemma \ref{lem:2}] Similarly to our argument for the global state model,
\begin{equation*}
 \bar{\theta}(t+1) = \bar{\theta}({t}) + \alpha_t \left[\bar{h}(t) + \bar{m}(t)\right].
\end{equation*}
Therefore, 
\begin{equation*}
 \left\|\bar{\theta}({t+1}) - \theta_{\rm lc}^*\right\|_2^2
 = \left\|\bar{\theta}({t}) - \theta_{\rm lc}^*\right\|_2^2 + 2 \alpha_t \left[\bar{h}(t) + \bar{m}(t)\right]^T (\bar{\theta}({t}) - \theta_{\rm lc}^*)
  + \alpha_t^2 \left\|\bar{h}(t) + \bar{m}(t)\right\|_2^2 .
\end{equation*}

Taking expectations:
\begin{equation}
  E\left[\left\|\bar{\theta}({t+1}) - \theta_{\rm lc}^*\right\|_2^2 \right] 
 = E\left[ \left\|\bar{\theta}({t}) - \theta_{\rm lc}^*\right\|_2^2 \right] + \alpha_t^2E\left[ \left\|\bar{h}(t) + \bar{m}(t)\right\|_2^2 \right] 
  - 2 \alpha_t E\left[ \left(\bar{h}(t)+ \bar{m}(t)\right)^T \left(\theta_{\rm lc}^* - \bar{\theta}({t})\right)\right] . \label{eq:exp_rec_lc}
\end{equation}

We consider the second term on the right hand side of Eq. (\ref{eq:exp_rec_lc}). Following the definition of $\bar{h}(t)$ and $\bar{m}(t)$, we have that
\begin{equation*}
 E\left[ \left\|\bar{h}(t) + \bar{m}(t)\right\|_2^2 \right] = E\left[ \left\|\frac{1}{N} \sum_{v \in \mathcal{V}} \delta_v(t)\phi(s_v(t)) \right\|_2^2\right].
\end{equation*}
Plugging in the expression for TD error $\delta_v(t)$ with Eq. (\ref{eq:delta_lc}), we obtain

\begin{align*}
 & E\left[ \left\|\bar{h}(t) + \bar{m}(t)\right\|_2^2 \right]\\
 =& E\left[ \left\| \frac{1}{N} \sum_{v \in \mathcal{V}} r_v(t) \phi(s_v(t)) - \frac{1}{N} \sum_{v \in \mathcal{V}} \phi(s_v(t))\left(\phi(s_v(t))- \gamma \phi(s_v'(t))\right)^T\theta_v(t) \right\|_2^2\right] \\
 =& E\left[ \left\| \frac{1}{N} \sum_{v \in \mathcal{V}} \left[ {r_v}(t) - \left(\phi(s_v(t))- \gamma \phi(s_v'(t))\right)^T \theta_{\rm lc}^*\right]\phi(s_v(t)) -\frac{1}{N} \sum_{v \in \mathcal{V}} \phi(s_v(t))\left(\phi(s_v(t))- \gamma \phi(s_v'(t))\right)^T( \theta_v(t) -\theta_{\rm lc}^*) \right\|_2^2\right] 
\end{align*}

Denote
\begin{equation}
    \bm{a}^* =\frac{1}{N} \sum_{v \in \mathcal{V}} \left[ {r_v}(t) - \left(\phi(s_v(t))- \gamma \phi(s_v'(t))\right)^T \theta_{\rm lc}^*\right]\phi(s_v(t)), \label{eq:a*}
\end{equation}
\begin{equation}
    \bm{b}^* = \frac{1}{N} \sum_{v \in \mathcal{V}} \phi(s_v(t))\left(\phi(s_v(t))- \gamma \phi(s_v'(t))\right)^T( \theta_v(t) -\theta_{\rm lc}^*). \label{eq:b*}
\end{equation}
Using inequality $\|\bm{a}^*-\bm{b}^*\|^2 \leq 2\|\bm{a}^*\|^2 + 2\|\bm{b}^*\|^2$, we obtain
\begin{equation}
 E\left[ \left\|\bar{h}(t) + \bar{m}(t)\right\|_2^2 \right] \leq 2E\left[\|\bm{a}^*\|^2\right] + 2E\left[\|\bm{b}^*\|^2\right]. \label{eq:2norm_lc}
\end{equation}

We first bound $E\left[\|\bm{a}^*\|^2\right]$. Let $\bm{a}^* =\frac{1}{N} \sum_{v \in \mathcal{V}} \rho_v,$ where $$\rho_v = \left[ {r_v}(t) - \left(\phi(s_v(t))- \gamma \phi(s_v'(t))\right)^T \theta_{\rm lc}^*\right]\phi(s_v(t)).$$Recall that, in the local state model, we are just running TD(0) on the identical MDP with the identical rewards across nodes. Hence the quantity $\theta_{\rm lc}^*$ satisfies Equation (\ref{eq:ls}), i.e., for all agents $v$,
 \begin{equation*}
 E \left[\left(r_v(t) - \left(\phi(s_v(t))-\gamma \phi(s_v'(t))\right)^T \theta_{\rm lc}^* \right) \phi(s_v(t)) \right] = 0.
 \end{equation*} We thus have that $ E [\rho_v] =0$ for $v \in \mathcal{V}$. Then
\begin{align*}
 E\left[\|\bm{a}^*\|^2\right] = & E\left[ \left( \bm{a}^* \right) ^T \bm{a}^*\right]\\
 = & E\left[ \left( \frac{1}{N} \sum_{v \in \mathcal{V}} \rho_v \right) ^T \left( \frac{1}{N} \sum_{v \in \mathcal{V}} \rho_v \right)\right] \\
 = & \frac{1}{N^2} E\left[ \sum_{v \in \mathcal{V}} \rho_v ^T \rho_v + \sum_{v \neq v'} \rho_v ^T \rho_{v'} \right] \\
 = & \frac{1}{N} E\left[ \rho_1 ^T \rho_1 \right] + \frac{1}{N^2}\sum_{v \neq v'} E[\rho_v]^T E[\rho_{v'}]\\
 =& \frac{1}{N} E\left[ \rho_1 ^T \rho_1 \right] = \frac{1}{N} E\left[ \| \rho_1 \|^2 \right],
\end{align*}where the forth line follows because we are assuming the quantities $\{s_v(t)\}_{v \in \mathcal{V}}$ are generated i.i.d. across time steps $t$ and the last line uses that $ E [\rho_v] =0$. Next,
\begin{align*}
  E\left[ \| \rho_1 \|^2 \right] 
 = & E\left[ \left \|\left(r_1(t) - \left(\phi(s_1(t))-\gamma \phi(s_1'(t))\right)^T \theta_{\rm lc}^* \right) \phi(s_1(t)) \right \|^2 \right] \\
 \leq & E\left[ \left( {r_1}(t) - \left(\phi(s_1(t))- \gamma \phi(s'_1(t))\right)^T \theta_{\rm lc}^*\right)^2 \right] = \sigma^2,
\end{align*}where, the inequality follows Assumption \ref{ass:features} and recall that ${\sigma}^2$ is defined by,
\begin{equation*}
 {\sigma}^2 = E\left[ \left( {r_v}(t) - \left(\phi(s_v(t))- \gamma \phi(s'_v(t))\right)^T \theta_{\rm lc}^*\right)^2 \right].
\end{equation*}

We have thus shown:
\begin{equation}
 E\left[\|\bm{a}^*\|^2\right] \leq \frac{{\sigma}^2}{N}. \label{eq:a_lc}
\end{equation}



Our next step is to bound $E\left[\|\bm{b}^*\|^2\right]$, where $\bm{b}^*$ is defined in Eq.(\ref{eq:b*}):

\begin{align}
    E\left[\|\bm{b}^*\|^2\right] =  & E\left[\left\| \frac{1}{N} \sum_{v \in \mathcal{V}} \left(\phi(s_v(t))- \gamma \phi(s_v'(t))\right)^T( \theta_v(t) -\theta_{\rm lc}^*) \phi(s_v(t)) \right\| ^2 \right] \notag\\ 
    = & \frac{1}{N^2} E\left[\left\| \sum_{v \in \mathcal{V}} \left(\phi(s_v(t))- \gamma \phi(s_v'(t))\right)^T( \theta_v(t) -\theta_{\rm lc}^*) \phi(s_v(t)) \right\| ^2 \right] \notag\\ 
    \leq & \frac{1}{N^2}\cdot N \cdot \sum_{v \in \mathcal{V}} E\left[ \left( \left(\phi(s_v(t))- \gamma \phi(s_v'(t))\right)^T( {\theta_v}(t) -\theta_{\rm lc}^*) \right)^2 \left\| \phi(s_v(t))\right\|^2\right] \notag\\
    \leq & \frac{1}{N^2}\cdot N \cdot \sum_{v \in \mathcal{V}} E\left[ \left( \left(\phi(s_v(t))- \gamma \phi(s_v'(t))\right)^T( {\theta_v}(t) -\theta_{\rm lc}^*) \right)^2 \right] \notag\\
    \leq & \frac{4}{N} \sum_{v \in \mathcal{V}} E\left[ \|V_{\theta_v(t)} - V_{\theta_{\rm lc}^*} \|_{D}^2\right], \label{eq:b_lc}
\end{align}where the first inequality uses  $\|\sum_{i=1}^N a_i x_i\|^2 \leq N \sum_{i=1}^N a_i^2 \| x_i\|^2$; the second inequality follows Assumption \ref{ass:features}; and the last line follows from the proof of Lemma 5 in \citet{bhandari2018finite}.

Plugging Eq. (\ref{eq:a_lc}) and Eq. (\ref{eq:b_lc}) into Eq. (\ref{eq:2norm_lc}), we get

\begin{equation}
 E\left[ \left\|\bar{h}(t) + \bar{m}(t)\right\|_2^2 \right] \leq \frac{2{\sigma}^2}{N} + \frac{8}{N} \sum_{v \in \mathcal{V}} E\left[ \|V_{\theta_v(t)} - V_{\theta_{\rm lc}^*} \|_{D}^2\right]. \label{eq:h+m_lc}
\end{equation}

This equation bounds bounds one of the terms in Eq. (\ref{eq:exp_rec_lc}). We next consider a different term in the same equation, namely we consider the third term on the right hand side of Eq. (\ref{eq:exp_rec_lc}):
\begin{align*}
  E\left[ \left[\bar{h}(t)+ \bar{m}(t)\right]^T (\theta_{\rm lc}^* - \bar{\theta}({t}))\right] 
 = & E\left[ E\left[\left(\bar{h}(t) + \bar{m}(t)\right)^T (\theta_{\rm lc}^* - \bar{\theta}({t}))| \Theta({t})\right] \right] \notag\\
 = & E\left[\bar{h}^T(t) (\theta_{\rm lc}^* - \bar{\theta}({t})) \right], 
\end{align*}
where in the second equation, we use Eq. (\ref{eq:con_exp_lc}).

By Proposition \ref{pro:bar_lc} part (a), we have that 
$\bar{h}(t) = {b} -A \bar{\theta}(t)$. Furthermore, if we let $ \bar{h}(\theta)$ denote the linear function ${b} -A {\theta}$, we have that $\bar{h}(\theta_{\rm lc}^*)=0$. Now applying Corollary 1 in \citet{liu2020temporal}, we have that
\begin{equation}
 E\left[ \left[\bar{h}(t)+ \bar{m}(t)\right]^T (\theta_{\rm lc}^* - \bar{\theta}({t}))\right] 
 = E\left[ (1-\gamma)\left\|V_{\theta_{\rm lc}^*} - V_{\bar{\theta}(t)} \right\|_{D}^2 + \gamma \left\|V_{\theta_{\rm lc}^*} - V_{\bar{\theta}(t)} \right\|_{{\rm Dir}}^2 \right] \label{eq:inner_product_lc}
\end{equation}

Combining equations (\ref{eq:exp_rec_lc}), (\ref{eq:inner_product_lc}), and (\ref{eq:h+m_lc}), we obtain
\begin{align*}
  E\left[\left\|\bar{\theta}({t+1}) - \theta_{\rm lc}^*\right\|_2^2 \right] \leq & E\left[ \left\|\bar{\theta}({t}) - \theta_{\rm lc}^*\right\|_2^2 \right] +\alpha_t^2 \left(\frac{2{\sigma}^2}{N} + \frac{8}{N} \sum_{v \in \mathcal{V}} E\left[ \|V_{\theta_v(t)} - V_{\theta_{\rm lc}^*} \|_{D}^2\right]\right )\\
 &  - 2 \alpha_t E\left[ (1-\gamma) \left\|V_{\theta_{\rm lc}^*} - V_{\bar{\theta}(t)} \right\|_{D}^2+ \gamma \left\|V_{\theta_{\rm lc}^*} - V_{\bar{\theta}(t)} \right\|_{{\rm Dir}}^2 \right].
\end{align*}
\end{proof}

With this lemma in place, we are now ready to provide a proof of Theorem \ref{thm:main}. This will be similar, but not identical, to the proof of Theorem \ref{thm:1}, as the recursion we have just proved as an extra term multiplying $O(\alpha_t^2)$ relative to Lemma \ref{lem:1}.

\begin{proof}[Proof of Theorem \ref{thm:main}]
Starting from Lemma \ref{lem:2},
\begin{align}
  E\left[\left\|\bar{\theta}({t+1}) - \theta_{\rm lc}^*\right\|_2^2 \right] \leq & E\left[ \left\|\bar{\theta}({t}) - \theta_{\rm lc}^*\right\|_2^2 \right]  +\alpha_t^2 \left(\frac{2{\sigma}^2}{N} + \frac{8}{N} \sum_{v \in \mathcal{V}} E\left[ \|V_{\theta_v(t)} - V_{\theta_{\rm lc}^*} \|_{D}^2\right]\right )\notag\\
 & - 2 \alpha_t E\left[ (1-\gamma) \left\|V_{\theta_{\rm lc}^*} - V_{\bar{\theta}(t)} \right\|_{D}^2+ \gamma \left\|V_{\theta_{\rm lc}^*} - V_{\bar{\theta}(t)} \right\|_{{\rm Dir}}^2 \right], \label{eq:rec2}
\end{align}

we first consider the bound for the term $\sum_{t=1}^T \sum_{v=1}^N E\left[ \|V_{\theta_v(t)} - V_{\theta_{\rm lc}^*} \|_{D}^2\right]$. We can plug in that $N=1$ into Lemma \ref{lem:2} to obtain the next inequality:
\begin{align*}
  E\left[\left\|{\theta}_v({t+1}) - \theta_{\rm lc}^*\right\|_2^2 \right] \leq & E\left[ \left\|{\theta}_v({t}) - \theta_{\rm lc}^*\right\|_2^2 \right]  +\alpha_t^2 \left(2{\sigma}^2 + 8 E\left[ \|V_{\theta_v(t)} - V_{\theta_{\rm lc}^*} \|_{D}^2\right]\right )\\
 &  - 2 \alpha_t E\left[ (1-\gamma) \left\|V_{\theta_{\rm lc}^*} - V_{{\theta}_v(t)} \right\|_{D}^2+ \gamma \left\|V_{\theta_{\rm lc}^*} - V_{{\theta}_v(t)} \right\|_{{\rm Dir}}^2 \right].
\end{align*}
If the sequence of step-sizes are non-increasing and satisfies $$8\alpha_t^2-2\alpha_t(1-\gamma) \leq - \alpha_t(1-\gamma),$$ then we obtain
\begin{equation*}
  \alpha_t E\left[ (1-\gamma) \left\|V_{\theta_{\rm lc}^*} - V_{{\theta}_v(t)} \right\|_{D}^2+ 2 \gamma \left\|V_{\theta_{\rm lc}^*} - V_{{\theta}_v(t)} \right\|_{{\rm Dir}}^2 \right]
 \leq  E\left[ \left\|{\theta}_v({t}) - \theta_{\rm lc}^*\right\|_2^2 \right] -E\left[\left\|{\theta}_v({t+1}) - \theta_{\rm lc}^*\right\|_2^2 \right] + 2 \alpha_t^2 {\sigma}^2.
\end{equation*}
Since $ E\left[ 2 \gamma \left\|V_{\theta_{\rm lc}^*} - V_{{\theta}_v(t)} \right\|_{{\rm Dir}}^2 \right]$ is non-negative, it now follows that
\begin{equation*}
  \alpha_t E\left[ (1-\gamma) \left\|V_{\theta_{\rm lc}^*} - V_{{\theta}_v(t)} \right\|_{D}^2 \right]
 \leq  E\left[ \left\|{\theta}_v({t}) - \theta_{\rm lc}^*\right\|_2^2 \right] -E\left[\left\|{\theta}_v({t+1}) - \theta_{\rm lc}^*\right\|_2^2 \right] + 2 \alpha_t^2 {\sigma}^2.
\end{equation*}
Multiplying $\alpha_t$ on both sides and summing over $t$, we have

\begin{align*}
  & \sum_{t=0}^{T-1} \alpha_t^2 E\left[ (1-\gamma) \left\|V_{\theta_{\rm lc}^*} - V_{{\theta}_v(t)} \right\|_{D}^2 \right] \\
 = & \alpha_0 E\left[ \left\|{\theta}_v({0}) - \theta_{\rm lc}^*\right\|_2^2 \right] + \sum_{t=1}^{T-1}(\alpha_{t-1} - \alpha_{t}) E\left[ \left\|{\theta}_v({t}) - \theta_{\rm lc}^*\right\|_2^2 \right]-\alpha_{T-1} E\left[\left\|{\theta}_v({T}) - \theta_{\rm lc}^*\right\|_2^2 \right] + 2 \sum_{t=0}^{T-1} \alpha_t^3 {\sigma}^2\\
 \leq &\alpha_0 E\left[ \left\|{\theta}_v({0}) - \theta_{\rm lc}^*\right\|_2^2 \right] + 2 \sum_{t=0}^{T-1} \alpha_t^3 {\sigma}^2,
\end{align*}where the last inequality is because that $\{\alpha_t\}_t$ are non-increasing step-sizes.
Summing over agents $v$ , we get
\begin{align}
  \sum_{v=1}^N \sum_{t=0}^{T-1} \alpha_t^2 E\left[ (1-\gamma) \left\|V_{\theta_{\rm lc}^*} - V_{{\theta}_v(t)} \right\|_{D}^2 \right] 
 \leq & \sum_{v=1}^N \alpha_0 E\left[ \left\|{\theta}_v({0}) - \theta_{\rm lc}^*\right\|_2^2 \right] + 2\sum_{v=1}^N \sum_{t=0}^{T-1} \alpha_t^3 {\sigma}^2 \notag \\
 \leq & N \alpha_0 \hat{R}_0 + 2 N \sum_{t=0}^{T-1} \alpha_t^3 {\sigma}^2, \label{eq:sum_vt}
\end{align}
where $\hat{R}_0 = \max_{v \in \mathcal{V}} E\left[ \left\|{\theta}_v({0}) - \theta_{\rm lc}^*\right\|_2^2 \right]. $
With this equation in place, we now turn to the proof of all the parts of the theorem.

\textbf{Proof of part (a):} We consider the constant step-size sequence $\alpha_0 = \cdots = \alpha_T \leq (1- \gamma)/8$. Then let $\alpha$ denote the constant step-size. Plugging into Eq. (\ref{eq:rec2}) and rearranging it, we get
\begin{align*}
  2 \alpha E\left[ (1-\gamma) \left\|V_{\theta_{\rm lc}^*} - V_{\bar{\theta}(t)} \right\|_{D}^2+ \gamma \left\|V_{\theta_{\rm lc}^*} - V_{\bar{\theta}(t)} \right\|_{{\rm Dir}}^2 \right]
 \leq &E\left[ \left\|\bar{\theta}({t}) - \theta_{\rm lc}^*\right\|_2^2 \right] - E\left[\left\|\bar{\theta}({t+1}) - \theta_{\rm lc}^*\right\|_2^2 \right]\\
 & +\alpha^2 \left(\frac{2{\sigma}^2}{N} + \frac{8}{N} \sum_{v \in \mathcal{V}} E\left[ \|V_{\theta_v(t)} - V_{\theta_{\rm lc}^*} \|_{D}^2\right]\right ).
\end{align*}
Summing over $t$ gives
\begin{align*}
 & 2 \sum_{t=0}^{T-1} \alpha E\left[ (1-\gamma) \left\|V_{\theta_{\rm lc}^*} - V_{\bar{\theta}(t)} \right\|_{D}^2+ \gamma \left\|V_{\theta_{\rm lc}^*} - V_{\bar{\theta}(t)} \right\|_{{\rm Dir}}^2 \right]\\
 \leq &E\left[ \left\|\bar{\theta}({0}) - \theta_{\rm lc}^*\right\|_2^2 \right] - E\left[\left\|\bar{\theta}({T}) - \theta_{\rm lc}^*\right\|_2^2 \right]+ \frac{2 T \alpha^2{\sigma}^2}{N} + \frac{8}{N} \sum_{t=0}^{T-1} \sum_{v \in \mathcal{V}} \alpha^2 E\left[ \|V_{\theta_v(t)} - V_{\theta_{\rm lc}^*} \|_{D}^2\right]\\
 \leq & E\left[ \left\|\bar{\theta}({0}) - \theta_{\rm lc}^*\right\|_2^2 \right] + \frac{2 T \alpha^2{\sigma}^2}{N} + \frac{8}{N} \sum_{t=0}^{T-1} \sum_{v \in \mathcal{V}} \alpha^2 E\left[ \|V_{\theta_v(t)} - V_{\theta_{\rm lc}^*} \|_{D}^2\right] \\
 \leq & E\left[ \left\|\bar{\theta}({0}) - \theta_{\rm lc}^*\right\|_2^2 \right] + \frac{2 T \alpha^2{\sigma}^2}{N} + \frac{8}{N(1-\gamma)} \left( N \alpha \hat{R}_0 + 2 N \sum_{t=0}^{T-1} \alpha^3 {\sigma}^2 \right)\\
 \leq & E\left[ \left\|\bar{\theta}({0}) - \theta_{\rm lc}^*\right\|_2^2 \right] + \frac{2 T \alpha^2{\sigma}^2}{N} + \frac{8\alpha}{1-\gamma} \left(\hat{R}_0 + 2 T \alpha^2 {\sigma}^2 \right)
\end{align*}where the second inequality follows that $E\left[\left\|\bar{\theta}({T}) - \theta_{\rm lc}^*\right\|_2^2 \right]$ is non-negative; the third inequality uses Eq. (\ref{eq:sum_vt}). 

Now dividing by $2 \alpha$ on both sides:

\begin{equation*}
  \sum_{t=0}^{T-1} E\left[ (1-\gamma) \left\|V_{\theta_{\rm lc}^*} - V_{\bar{\theta}(t)} \right\|_{D}^2+ \gamma \left\|V_{\theta_{\rm lc}^*} - V_{\bar{\theta}(t)} \right\|_{{\rm Dir}}^2 \right]
 \leq  \frac{1}{2\alpha}E\left[ \left\|\bar{\theta}({0}) - \theta_{\rm lc}^*\right\|_2^2 \right] + \frac{ T \alpha {\sigma}^2}{N} + \frac{4}{1-\gamma} \left(\hat{R}_0 + 2 T \alpha^2 {\sigma}^2 \right).
\end{equation*}Let $\hat{\theta}(T)=\frac{1}{T} \sum_{t=1}^T \bar{\theta}(t) $. Then, by convexity
\begin{align*}
  E\left[(1-\gamma) \left\|V_{\theta_{\rm lc}^*} - V_{\hat{\theta}(T)} \right\|_{D}^2 + \gamma \left\|V_{\theta_{\rm lc}^*} - V_{\hat{\theta}(T)} \right\|_{{\rm Dir}}^2 \right] 
 \leq& \frac{1}{T} \sum_{t=1}^T E\left[ (1-\gamma) \left\|V_{\theta_{\rm lc}^*} - V_{\bar{\theta}(t)} \right\|_{D}^2 + \gamma \left\|V_{\theta_{\rm lc}^*} - V_{\bar{\theta}(t)} \right\|_{{\rm Dir}}^2 \right]\\
 \leq & \frac{1}{ T } \left(\frac{1}{2\alpha} E\left[ \left\|\bar{\theta}({0}) - \theta_{\rm lc}^*\right\|_2^2 \right] +\frac{4\hat{R}_0}{1-\gamma} \right)+ \frac{ \alpha {\sigma}^2}{N} + \frac{8 \alpha^2 {\sigma}^2}{1-\gamma},
\end{align*} which is what we wanted to show. 

\textbf{Proof of part (b):} We now consider the step-size $\alpha_0 = \cdots = \alpha_T =\frac{1}{\sqrt{T}}$. When $T \geq \frac{64}{ (1-\gamma)^2}$, it can be observed that $\alpha = \frac{1}{\sqrt{T}} \leq \frac{1-\gamma}{8}$. As a consequence of part (a), it is immediate that, 
\begin{equation*}
 E\left[(1-\gamma) \left\|V_{\theta_{\rm lc}^*} - V_{\hat{\theta}(T)} \right\|_{D}^2 + \gamma \left\|V_{\theta_{\rm lc}^*} - V_{\hat{\theta}(T)} \right\|_{{\rm Dir}}^2 \right] 
 \leq  \frac{1}{2 \sqrt{T} }\left( E\left[ \left\|\bar{\theta}({0}) - \theta_{\rm lc}^*\right\|_2^2 \right]+\frac{ 2 {\sigma}^2}{N} \right) +\frac{1}{T} \left(\frac{4\hat{R}_0+ 8{\sigma}^2 }{1-\gamma} \right),
\end{equation*} which is what we wanted to show.

\textbf{Proof of part (c):} Using that $\gamma \left\|V_{\theta_{\rm lc}^*} - V_{\bar{\theta}(t)} \right\|_{{\rm Dir}}^2$ is non-negative and rearranging Eq. (\ref{eq:rec2}), we have

\begin{equation*}
  E \left[ \left\|\bar{\theta}({t+1}) - \theta_{\rm lc}^*\right\|_2^2 \right] \leq E\left[ \left\|\bar{\theta}({t}) - \theta_{\rm lc}^*\right\|_2^2 \right]
 +\alpha_t^2 \left(\frac{2{\sigma}^2}{N} + \frac{8}{N} \sum_{v \in \mathcal{V}} E\left[ \|V_{\theta_v(t)} - V_{\theta_{\rm lc}^*} \|_{D}^2\right]\right )
  - 2 \alpha_t (1-\gamma) E \left\|V_{\theta_{\rm lc}^*} - V_{\bar{\theta}(t)} \right\|_D^2 .
\end{equation*}

Applying Lemma 1 in \citet{bhandari2018finite}, which states that $$\sqrt{\omega} \|\theta\|_2 \leq \|V_{\theta}\|_{D} \leq \|\theta\|_2 ,$$ we get
\begin{equation}
  E \left[ \left\|\bar{\theta}({t+1}) - \theta_{\rm lc}^*\right\|_2^2 \right] \leq ( 1-2 \alpha_t (1-\gamma) \omega) E\left[ \left\|\bar{\theta}({t}) - \theta_{\rm lc}^*\right\|_2^2 \right] 
 + \alpha_t^2 \left(\frac{2{\sigma}^2}{N} + \frac{8}{N} \sum_{v \in \mathcal{V}} E\left[ \|V_{\theta_v(t)} - V_{\theta_{\rm lc}^*} \|_{D}^2\right]\right) \label{eq:1/t}.
\end{equation}

We first consider the last term on the right hand side, i.e., $E\left[ \|V_{\theta_v(t)} - V_{\theta_{\rm lc}^*} \|_{D}^2\right]$. Since each agent in the system executes the classical TD($0$) at time $t$ for $t \in \mathbb{N}_0$, then by part (c) of Theorem 2 and Lemma 1 in \citep{bhandari2018finite}, for $v \in \mathcal{V}$, we have that
\begin{equation*}
 E\left[ \|V_{\theta_v(t)} - V_{\theta_{\rm lc}^*} \|_{D}^2\right] 
 \leq E\left[ \|{\theta_v(t)} - {\theta_{\rm lc}^*} \|_{2}^2\right] 
 \leq \frac{\hat{\zeta}}{t + \tau},
\end{equation*}where $$\hat{\zeta} = \max\left\{ {2 \alpha^2 {\sigma}^2}, \tau \hat{R}_0 \right\},$$ recall that $\hat{R}_0 = \max_{v \in \mathcal{V}} E\left[ \left\|{\theta}_v({0}) - \theta_{\rm lc}^*\right\|_2^2 \right] $ and $\tau = \frac{16}{(1-\gamma)^2\omega}$.
Hence,
\begin{equation*}
 \frac{8}{N} \sum_{v \in \mathcal{V}} E\left[ \|V_{\theta_v(t)} - V_{\theta_{\rm lc}^*} \|_{D}^2\right] \leq \frac{8 \hat{\zeta}}{t + \tau},
\end{equation*}and plugging it into Eq. (\ref{eq:1/t}), we can obtain

\begin{align*}
  E \left[ \left\|\bar{\theta}({t+1}) - \theta_{\rm lc}^*\right\|_2^2 \right]
 \leq & ( 1-2 \alpha_t (1-\gamma) \omega) E\left[ \left\|\bar{\theta}({t}) - \theta_{\rm lc}^*\right\|_2^2 \right] + \alpha_t^2 \left(\frac{2{\sigma}^2}{N} + \frac{8 \hat{\zeta}}{t + \tau} \right) \notag\\
 = & \left( 1 - \frac{4}{t+\tau}\right)E\left[ \left\|\bar{\theta}({t}) - \theta_{\rm lc}^*\right\|_2^2 \right] + \frac{2\alpha ^2 {\sigma}^2/N}{(t+\tau)^2} + \frac{8 \alpha ^2 \hat{\zeta}}{(t+\tau)^3}, 
\end{align*}where we use that $\alpha_t = \frac{\alpha}{t+\tau}$ with $\alpha = \frac{2}{(1-\gamma)\omega}$ and $\tau = \frac{16}{(1-\gamma)^2\omega}$ to get the last line. This recursion implies that
\begin{align}
 E \left[ \left\|\bar{\theta}({t+1}) - \theta_{\rm lc}^*\right\|_2^2 \right] 
 \leq & \prod_{i=0}^t \left( 1- \frac{4}{t+\tau-i}\right) E \left[ \left\|\bar{\theta}({0}) - \theta_{\rm lc}^*\right\|_2^2 \right] + \frac{2\alpha ^2 {\sigma}^2}{N} \sum_{i=0}^t \left[\frac{1}{(t+\tau-i)^2} \prod_{l=0}^{i-1} \left( 1- \frac{4}{t+\tau-l}\right) \right] \notag\\
 & + 8 \alpha ^2 \hat{\zeta} \sum_{i=0}^t \left[\frac{1}{(t+\tau-i)^3} \prod_{l=0}^{i-1} \left( 1- \frac{4}{t+\tau-l}\right) \right].\label{eq:rec_0}
\end{align}
Consider the product
\begin{align}
  \prod_{i=0}^t \left( 1- \frac{4}{t+\tau-i}\right) 
 = & \frac{t+\tau-4}{t+\tau} \cdot \frac{t+\tau-5}{t+\tau-1} \cdot \cdots \cdot \frac{\tau-4}{\tau}\notag \\
 =& \frac{(\tau-1)(\tau-2)(\tau-3)(\tau-4)}{(t+\tau)(t+\tau-1)(t+\tau-2)(t+\tau-3)} \notag \\
 =& \frac{\tau-1}{t+\tau} \cdot \frac{\tau-2}{t+\tau-1} \cdot \frac{\tau-3}{t+\tau-2} \cdot\frac{\tau-4}{t+\tau-3} \notag \\
 < & \left( \frac{\tau-1}{t+\tau} \right)^4. \label{eq:prod}
\end{align}The last inequality follows because that last three terms in equation is smaller than $(\tau-1)/(t+\tau)$. Indeed, for $i =1,2,3$, we have that 
\begin{align*}
 \frac{\tau -i-1}{t+\tau-i} = & \frac{\tau-1}{t+\tau} + \left( \frac{\tau -i-1}{t+\tau-i} - \frac{\tau-1}{t+\tau} \right) \\
 = & \frac{\tau-1}{t+\tau} + \frac{(t+\tau)(\tau -i-1) - (\tau-1) (t+\tau-i)}{(t+\tau)(t+\tau-i)}\\
 =& \frac{\tau-1}{t+\tau} - \frac{(i+1)t }{(t+\tau)(t+\tau-i)}\\
 < & \frac{\tau-1}{t+\tau}.
\end{align*} 

For the other product $\prod_{l=0}^{i-1} \left( 1- \frac{4}{t+\tau-l}\right) $ in Eq. (\ref{eq:rec_0}), applying the same method of Eq. (\ref{eq:prod}), we thus have
\begin{equation}
 \prod_{l=0}^{i-1} \left( 1- \frac{4}{t+\tau-l}\right) \leq \left( \frac{t+\tau-i}{t+\tau} \right)^4. \label{eq:prod2}
\end{equation}

Using Eq. (\ref{eq:prod}) and Eq. (\ref{eq:prod2}), Eq. (\ref{eq:rec_0}) becomes 
\begin{align}
 & E \left[ \left\|\bar{\theta}({t+1}) - \theta_{\rm lc}^*\right\|_2^2 \right] \notag \\
 \leq & \left( \frac{\tau-1}{t+\tau} \right)^4 E \left[ \left\|\bar{\theta}({0}) - \theta_{\rm lc}^*\right\|_2^2 \right] + \frac{2\alpha ^2 {\sigma}^2}{N} \sum_{i=0}^t \left[\frac{1}{(t+\tau-i)^2} \left( \frac{t+\tau-i}{t+\tau} \right)^4 \right]  + 8 \alpha ^2 \hat{\zeta} \sum_{i=0}^t \left[\frac{1}{(t+\tau-i)^3} \left( \frac{t+\tau-i}{t+\tau} \right)^4 \right] \notag\\
 \leq & \left( \frac{\tau-1}{t+\tau} \right)^4 E \left[ \left\|\bar{\theta}({0}) - \theta_{\rm lc}^*\right\|_2^2 \right]  + \frac{2 \alpha ^2 {\sigma}^2}{N} \sum_{i=0}^t \frac{(t+\tau-i)^2}{(t+\tau)^4} + 8 \alpha ^2 \hat{\zeta} \sum_{i=0}^t \frac{t+\tau-i}{(t+\tau)^4}. \label{eq:lcc1}
\end{align}Next, we bound summations of the second and the third on the right hand side $\sum_{i=0}^t \frac{(t+\tau-i)^2}{(t+\tau)^4} $ and $\sum_{i=0}^t \frac{t+\tau-i}{(t+\tau)^4} $ separately. For the summation in the second term, i.e., $\sum_{i=0}^t \frac{(t+\tau-i)^2}{(t+\tau)^4} $, we have
\begin{align}
  \sum_{i=0}^t \frac{(t+\tau-i)^2}{(t+\tau)^4} 
 = & \sum_{i=0}^t \frac{(i+\tau)^2}{(t+\tau)^4} \notag\\
 \leq & \frac{1}{(t+\tau)^4} \sum_{i=1}^{t+\tau} i^2\notag\\
 = & \frac{1}{(t+\tau)^4}\cdot \frac{1}{6} \cdot (t+\tau)(t+\tau+1)(2t+2\tau+1)\notag\\
 \leq & \frac{1}{6(t+\tau)^4} (t+\tau) \left( 2(t+\tau) \right)\left( 3(t+\tau) \right)\notag\\
 = & \frac{1}{t+\tau}.\label{eq:lcc2}
\end{align}For the summation in the third term, i.e., $\sum_{i=0}^t \frac{t+\tau-i}{(t+\tau)^4} $, it is immediately that
\begin{align}
  \sum_{i=0}^t \frac{t+\tau-i}{(t+\tau)^4} 
 = & \sum_{i=0}^t \frac{i+\tau}{(t+\tau)^4}\notag\\
 \leq & \frac{1}{(t+\tau)^4} \sum_{i=0}^{t} (i+\tau)\notag\\
 = & \frac{1}{(t+\tau)^4}\cdot \frac{(t+1)(t+2\tau)}{2} \notag\\
 \leq & \frac{1}{2(t+\tau)^4} \left( (t+\tau) \right)\left( 2(t+\tau) \right)\notag\\
 = & \frac{1}{(t+\tau)^2}.\label{eq:lcc3}
\end{align}Therefore, combining Eq. (\ref{eq:lcc1}), Eq. (\ref{eq:lcc2}), and Eq. (\ref{eq:lcc3}), we obtain
\begin{equation*}
  E \left[ \left\|\bar{\theta}({t+1}) - \theta_{\rm lc}^*\right\|_2^2 \right] \leq \frac{2 \alpha ^2 {\sigma}^2/N}{t+\tau} + \frac{8 \alpha ^2 \hat{\zeta} }{(t+\tau)^2}+ \frac{(\tau-1)^4E \left[ \left\|\bar{\theta}({0}) - \theta_{\rm lc}^*\right\|_2^2 \right]}{(t+\tau)^4} .
\end{equation*}
\end{proof}

\section{Numerical Experiments}
In this section, we provide details of the simulations done in the main body of the paper. These simulations were done on OpenAI control problems and GridWorld. We first give the details of the Gridworld setup, which is fairly standard. 

\subsection{Settings on the Gridworld MDP}\label{subsec:gw}
In this subsection, we introduce the specific problem settings for the grid-world MDP. We consider a $4 \times 4$ grid, where the states are $\mathcal{S} = \{1, 2, \ldots, 16\}$. There are four possible actions for each state,  $\mathcal{A}$ = \{left, right, up, down\}. If the action leads out of the grid, then the next state will remain to be the current state. Set the discount factor in the MDP to be $0.8$, $\gamma = 0.8$. Let deterministic rewards $r(s,a)$ be randomly chosen from a normal distribution $\mathcal{N}(1,100)$.
\begin{table}[h]
    \centering
    \begin{tabular}{|c|c|c|c|}
\hline  
1&2&3&4\\
\hline  
5&6&7&8\\
\hline 
9&10&11&12\\
\hline
13&14&15&16\\
\hline
\end{tabular}
    \caption{State space of grid world}
    \label{tab:my_label}
\end{table}

In this experiment, we will consider a random policy, i.e., each agent chooses an action from the 4 possible actions uniformly at random. Feature vectors are generated as $\phi(s) = (1,0,0,0)^T$ for four upper-left states $s \in \{1,2,5,6\}$; $\phi(s) = (0,1,0,0)^T$ for four upper-right states $s \in \{3,4,7,8\}$; $\phi(s) = (0,0,1,0)^T$ for four lower-left states $s \in \{9,10,13,14\}$; and $\phi(s) = (0,0,0,1)^T$ for four lower-right states $s \in \{11,12,15,16\}$. In this case, for any parameter $\theta \in \mathbb{R}^4$, we have $V_{\theta}(s) = \theta^{T}\phi(s)$ as the approximation for the value function of state $s$. Furthermore, samples are generated i.i.d and are equally likely chosen from the state space.

Due to the relatively small state space and the fixed policy, it is simply for us to use both the  transition matrix $P$, and further get stationary distribution $\pi$. We can also get $\theta_{\rm lc}^*$ by solving Eq.(\ref{eq:ls}). Therefore, the left-hand side of Theorem \ref{thm:main}(a):
\begin{equation*}
    E\left[(1-\gamma) \left\|V_{\theta_{\rm lc}^*} - V_{\hat{\theta}(T)} \right\|_{D}^2 + \gamma \left\|V_{\theta_{\rm lc}^*} - V_{\hat{\theta}(T)} \right\|_{{\rm Dir}}^2 \right],
\end{equation*} where recall that norm $\| \cdot\|_{D}$ and semi-norm $\| \cdot\|_{\rm Dir}$ is defined as Eq.(\ref{eq:def_D}) and Eq.(\ref{eq:def_Dir}) with stationary distribution $\pi$, can be obtained exactly for any constant step-size.

\subsection{Settings on the Classic Control Problems}\label{subsec:control}

Unlike the grid world case, for a more involved RL problem we do not have an explicit solution for the final limit $\theta_{\rm lc}^*$ that we can compare to. In other words, it is not possible to plot the left-hand side of Theorem \ref{thm:main} which contains the optimal parameter vector $\theta_{\rm lc}^*$. As a consequence, we use the empirical variances among several runs of the method, which is a plausible measure for accuracy of the method in place of the left-hand side of Theorem \ref{thm:main}(a).

For classic control problems, we use the tile coding \cite{sutton2018reinforcement} to deal with multi-dimensional continuous spaces. A tiling is a partition of the state space and a tile is an element of a partition. We set the parameter dimension $K$ be the total number of tiles among all tilings. The feature vector of state $s$, $\phi(s) \in \mathbb{R}^K$, is a vector has one component for each tile in each tiling. For a state $s$, it falls in exactly one tile for each tiling. The element in the $\phi(s)$ corresponding  to the tile that $s$ falls within is one and all others are zeros. Hence, the number of ones in the feature vector is always equal to the number of tilings.

The numbers of tiling and grid are similar to those used in \citet{lakshminarayanan2018linear}. We use 5 tilings, and each tiling has $7 \times 7$ grids for two dimensional MountainCar-v1 and MountainCarContinous-v0; $3 \times 3 \times 3$ grids for three dimensional Pendulum-v1; $3^4$ grids for four dimensional CartPole-v1; and $2^6$ grids for six dimensional Acrobot-v1. We considered uniform random policy for all problems. The discount factor was $0.8$. The initial condition $\theta_v(0)$ were sampled form standard normal distribution and for a fixed initialization. We applied Algorithm \ref{algo:our} several times and then computed the empirical variance in the final estimates. As in the previous subsection, the step-size were chosen to be constant. In Figure \ref{fig:var}, each subplot shows the empirical variance for many different choices of $\alpha$ with $N=1$ and $N = \frac{1-\gamma}{8 \alpha}$.  As we expected, all the blue lines for $N=\frac{1-\gamma}{8 \alpha}$ are approximately quadratic in shapes while all the red lines are generally linear, consistent with our theoretical results. 


\subsection{TD Errors of Distributed TD Methods}

We now discuss the details of the simulations that generated Figure \ref{fig:compare_six}, the comparison of our Algorithm \ref{algo:our} with earlier distributed TD methods from \citet{doan2019finite} and \citet{wang2020decentralized}. We plot the averaged TD error among the network, i.e., $\frac{1}{N} \sum_{v \in \mathcal{V}} {\delta}_v({t})$ vs iteration $t$ on the x-axis.

The number of agents $N = 100$. For the distributed TD algorithms proposed in \citet{doan2019finite} and \citet{wang2020decentralized}, the communication graph among agents is generated by the Erdos–Renyi model, which is connected. 
In the grid world case, all the settings are the same as stated in subsection \ref{subsec:gw} except that deterministic rewards $r(s,a)$ be randomly chosen from a normal distribution $\mathcal{N}(1,0.01)$. The step-size is constant as $\alpha= 0.3$. The parameters selection is mainly based on the parameters used in the simulations of \citet{wang2020decentralized}.

For two dimensional MountainCar-v1 and MountainCarContinous-v0, we used 5 tilings, each tiling has $7 \times 7$ grids, and step-size $\alpha = 0.3$. For three dimensional Pendulum-v1, we used 5 tilings, each tiling has $5 \times 5 \times 5$ grids, and step-size $\alpha = 0.05$. For four dimensional CartPole-v1, we used 5 tilings, each tiling has $8^4$ grids, and step-size $\alpha = 0.3$. For six dimensional Acrobot-v1, we used 5 tilings, each tiling has $2^6$ grids, and step-size $\alpha = 0.05$. The constant step sizes for open AI gym problems are chosen from the the set $\Lambda = \{0.3, 0.25, 0.2, \cdots, 0.05\}$. For each problem, we choose the largest step size from the set $\Lambda$ such that all methods converge or the smallest of these step sizes $0.05$ even if there exists one method does not converge for all step sizes in the set $\Lambda$. Note that the experiments of CartPole and Pendulum, the method in \citet{wang2020decentralized} does not converge with any step sizes in the set $\Lambda$; but Algorithm \ref{algo:our} in this paper and method in \citet{doan2019finite} do converge with all step sizes in the set $\Lambda$. We only show experimental result with $\alpha =0.05$ in the main text.


\end{document}